\documentclass[conference]{IEEEtran}

\usepackage{etoolbox}
\newtoggle{preprint}
    \makeatletter
\def\@fnsymbol#1{\ensuremath{\ifcase#1\or \star\or \dagger\or \ddagger\or
   \mathsection\or \mathparagraph\or \|\or **\or \dagger\dagger
   \or \ddagger\ddagger \else\@ctrerr\fi}}
    \makeatother

\iftoggle{preprint}{
    \usepackage{charter}
    \usepackage[margin=1.25in]{geometry}
    \emergencystretch 3em
    \usepackage{amsmath, amsthm}
    \date{
    \emph{\small Department of Electrical and Computer Engineering$^1$} \\
    \emph{\small Department of Operations Research and Financial Engineering$^2$} \\
    \emph{\small Center for Statistics and Machine Learning$^3$} \\
    \emph{\small Department of Computer Science$^4$} \\
    \emph{\small Princeton University}}
    \author{\small 
    Haimin Hu$^{1,}$\thanks{H.\ Hu and G. Dragotto contributed equally.}$\;\;$ \orcidA{} Gabriele Dragotto$^{2,3,\star}$ \orcidB{} Zixu Zhang$^1$ \orcidC{} Kaiqu Liang$^4$ \orcidD{} \\ \small Bartolomeo Stellato$^{1,2,3}$ \orcidE{} Jaime Fernández Fisac$^{1,3,4}$ \orcidF{}
    }

}{
    \usepackage{times}
    \author{
        \IEEEauthorblockN{Haimin Hu\IEEEauthorrefmark{1}\textsuperscript{\textsection},
        Gabriele Dragotto\IEEEauthorrefmark{2}\textsuperscript{\textsection},
        Zixu Zhang\IEEEauthorrefmark{1},
        Kaiqu Liang\IEEEauthorrefmark{3},
        Bartolomeo Stellato\IEEEauthorrefmark{1}\IEEEauthorrefmark{2},
        Jaime Fernández Fisac\IEEEauthorrefmark{1}\IEEEauthorrefmark{3}
        }
        \IEEEauthorblockA{\IEEEauthorrefmark{1}Department of Electrical and Computer Engineering, Princeton University, United States
        }
        \IEEEauthorblockA{\IEEEauthorrefmark{2}Department of Operations Research and Financial Engineering, Princeton University, United States
        }
        \IEEEauthorblockA{\IEEEauthorrefmark{3}Department of Computer Science, Princeton University, United States
        }
        Email: \texttt{\{\href{mailto:haiminh@princeton.edu}{haiminh},
                         \href{mailto:gdragotto@princeton.edu}{gdragotto},
                         \href{mailto:zixuz@princeton.edu}{zixuz},
                         \href{mailto:kl2471@princeton.edu}{kl2471},
                         \href{mailto:bstellato@princeton.edu}{bstellato},
                         \href{mailto:jfisac@princeton.edu}{jfisac}\}@princeton.edu}
    }
}

\usepackage[dvipsnames]{xcolor}
\definecolor{porange}{HTML}{E77500} % Princeon Orange
\usepackage[numbers,sort&compress]{natbib}
\usepackage{multicol}
\usepackage[colorlinks,
            linkcolor=BlueViolet,
            citecolor=blue,
            urlcolor=porange,
            bookmarks]{hyperref}
\usepackage{float}
\usepackage{amsfonts}
\usepackage{amsthm}
\usepackage{fancyhdr}
\usepackage{comment}
\usepackage{amsmath}
\usepackage{changepage}
\usepackage{amssymb}
\usepackage{graphicx}
\usepackage{adjustbox}
\usepackage{latexsym}
\usepackage{enumerate}
\usepackage{cleveref}
\usepackage{float}
\usepackage[labelfont=bf,font=small]{caption}
\usepackage{mathtools}
\usepackage[ruled,vlined,linesnumbered]{algorithm2e}
\usepackage{bbm}
\usepackage{bm}
\usepackage{booktabs}
\usepackage[xindy, acronyms, nowarn]{glossaries}   % 
\usepackage{multirow}
\usepackage{makecell}
\usepackage{colortbl}
\usepackage{balance}

\usepackage{textcomp}
\usepackage{stfloats}  % Used to help placement of figure*

\usepackage{graphicx, pgfplots, subcaption, standalone, pgf, tikz}
\definecolor{lime}{HTML}{A6CE39}
\DeclareRobustCommand{\orcidicon}{
    \hspace{-3mm}
	\begin{tikzpicture} 
    \draw[lime, fill=lime] (0,0) circle [radius=0.15] node[white] { 
        {\fontfamily{qag}\selectfont \tiny ID} 
    };
	\end{tikzpicture} 
    \hspace{-2mm}
}

\foreach \x in {A, ..., Z}{\expandafter\xdef\csname orcid\x\endcsname{\noexpand\href{https://orcid.org/\csname orcidauthor\x\endcsname}{\noexpand\orcidicon} } }
\usepackage{amsthm}
\newtheorem{theorem}{Theorem}
\newtheorem{lemma}{Lemma}    

\newtheorem{remark}{Remark}
\newtheorem{example}{Example}    

\newtheorem{proposition}{Proposition}    
\newtheorem{assumption}{Assumption}    
\newtheorem{definition}{Definition}

\definecolor{porange}{HTML}{E77500} % Princeon Orange

\DeclareDocumentEnvironment{example}{}{\noindent\textbf{Running example:}\itshape}{}

\usepackage{ifthen}
\newboolean{include-notes}
\newboolean{include-new}
\newboolean{include-remove}
\setboolean{include-notes}{false}
\setboolean{include-new}{false}
\setboolean{include-remove}{false}

\usepackage{xcolor}
\usepackage[normalem]{ulem}
\newcommand{\jaime}[1]{\ifthenelse{\boolean{include-notes}}{\textcolor{orange}{\textbf{Jaime:} #1}}{}}
\newcommand{\haimin}[1]{\ifthenelse{\boolean{include-notes}}{\textcolor{magenta}{\textbf{Haimin:} #1}}{}}
\newcommand{\gab}[1]{\ifthenelse{\boolean{include-notes}}{\textcolor{cyan}{\textbf{Gabriele:} #1}}{}}
\newcommand{\bart}[1]{\ifthenelse{\boolean{include-notes}}{\textcolor{teal}{\textbf{Bartolomeo:} #1}}{}}
\newcommand{\remove}[1]{\ifthenelse{\boolean{include-remove}}{\textcolor{red}{\sout{#1}}}{}}
\newcommand{\new}[1]{\ifthenelse{\boolean{include-new}}{\textcolor{blue}{#1}}{#1}}
\newcommand{\todo}[1]{\ifthenelse{\boolean{include-notes}}{\textcolor{blue}{\textbf{TODO:} #1}}{}}

\newcommand{\p}[1]{\smallskip \noindent \textbf{{#1}.}}

\newcommand{\eg}{\emph{e.g.}}
\newcommand{\ie}{\emph{i.e.}}

\newcommand{\circledtext}[1]{\raisebox{.5pt}{\textcircled{\raisebox{-.9pt} {#1}}}}

\newcommand{\reals}{\mathbb{R}}

\DeclareMathOperator*{\argmin}{arg\,min}

\newcommand{\st}{\text{s.t.}}

\newcommand{\intvar}{z}

\newcommand{\nagents}{{N}}

\newcommand{\state}{{x}}
\newcommand{\bstate}{\bar{\state}}

\newcommand{\ctrl}{{u}}

\newcommand{\traj}{{\mathbf{x}}}%{{\xi}}%
\newcommand{\btraj}{\bar{\traj}}
\newcommand{\ttraj}{\tilde{\traj}}

\newcommand{\cset}{{\mathcal{U}}}

\newcommand{\dyn}{{f}}

\newcommand{\policy}{{\gamma}}

\newcommand{\tpolicy}{\tilde{\policy}}
\newcommand{\policyset}{{\Gamma}}
\newcommand{\pset}{{\policyset}}
\newcommand{\tpolicyset}{\tilde{\policyset}}
\newcommand{\bpolicyset}{\bar{\policyset}}

\newcommand{\perm}{p}
\newcommand{\permnode}{\perm^\node}
\newcommand{\permset}{{P}}
\newcommand{\permincomplete}{{\tilde{p}}}
\newcommand{\permincompleteassigned}{\permincomplete_+}
\newcommand{\permincompleteunassigned}{\permincomplete_-}
\newcommand{\permorder}{{\rho}}
\newcommand{\permsetincomplete}{\tilde{P}}
\newcommand{\permmapping}{\Phi}
\newcommand{\orderset}{\mathcal{I}}
\newcommand{\nodelist}{\mathcal{S}}
\newcommand{\node}{s}

\newcommand{\resmap}{R}

\newcommand{\eqm}{\mathcal{E}}

\newcommand{\preset}{{\mathcal{L}}}
\newcommand{\sucset}{{\mathcal{M}}}

\newcommand{\cost}{{J}}
\newcommand{\tcost}{\tilde{\cost}}
\newcommand{\costsafe}{\cost_{\rm{safe}}}
\newcommand{\tcostsafe}{\tilde{\cost}_{\rm{safe}}}
\newcommand{\costindi}{\cost_{\rm{indv}}}
\newcommand{\costsocial}{{\mathbf{J}}}
\newcommand{\stagecost}{{g}}
\newcommand{\stagecostsafe}{{\ell}}
\newcommand{\stagecostper}{{\bar{g}}}

\crefname{algocf}{Line}{Lines}
\crefname{lemma}{Lemma}{Lemmata}
\crefname{theorem}{Theorem}{Theorems}
\crefname{proposition}{Proposition}{Propositions}
\crefname{algorithm}{Algorithm}{Algorithms}
\crefname{equation}{}{}
\crefname{definition}{Definition}{Definition}
\crefname{claim}{Claim}{Claim}
\crefname{corollary}{Corollary}{Corollaries}
\crefname{remark}{Remark}{Remarks}
\crefname{example}{Example}{Examples}
\crefname{figure}{Figure}{Figures}
\crefname{section}{Section}{Sections}
\crefname{table}{Table}{Tables}

\newglossarystyle{mylist}{%
  \setglossarystyle{list}% base style is list
}

\glsdisablehyper
\makeglossaries
\newglossaryentry{LSE}
{
  name={LSE},
  plural={LSEs},
  description={local Stackelberg equilibrium},
  first={local Stackelberg equilibrium (\glsentrytext{LSE})},
  descriptionplural={local Stackelberg equilibria},
  firstplural={local Stackelberg equilibria (\glsentryplural{LSE})}
}

\newglossaryentry{GSE}
{
  name={GSE},
  description={global Stackelberg equilibrium},
  first={global Stackelberg equilibrium (\glsentrytext{GSE})}
}

\newglossaryentry{BNP}
{
  name={B\&P},
  description={Branch-and-Play},
  first={Branch-and-Play (\glsentrytext{BNP})},
}

\newglossaryentry{LQ}
{
  name={LQ},
  description={linear quadratic},
  first={linear quadratic (\glsentrytext{LQ})},
}

\newglossaryentry{ILQR}
{
  name={ILQR},
  description={iterative linear quadratic regulator},
  first={iterative linear quadratic regulator (\glsentrytext{ILQR})},
}

\newglossaryentry{ATC}
{
  name={ATC},
  description={air traffic control},
  first={air traffic control (\glsentrytext{ATC})},
}

\newglossaryentry{STP}
{
  name={STP},
  description={sequential trajectory planning},
  first={sequential trajectory planning (\glsentrytext{STP})},
}

\newglossaryentry{FCFS}
{
  name={FCFS},
  description={first-come-first-served},
  first={first-come-first-served (\glsentrytext{FCFS})},
}

\newglossaryentry{MAPF}
{
  name={MAPF},
  description={multi-agent pathfinding},
  first={Multi-agent pathfinding (\glsentrytext{MAPF})},
}

\newglossaryentry{MPC}
{
  name={MPC},
  description={model predictive control},
  first={model predictive control (\glsentrytext{MPC})},
}

\begin{document}
\title{Who Plays First? Optimizing the Order of Play in Stackelberg Games with Many Robots}

\makeatletter
\let\@oldmaketitle\@maketitle% Store \@maketitle
\renewcommand{\@maketitle}{\@oldmaketitle% Update \@maketitle to insert...
% \vspace{-0.2in}
\centering
\includegraphics[width=\textwidth]{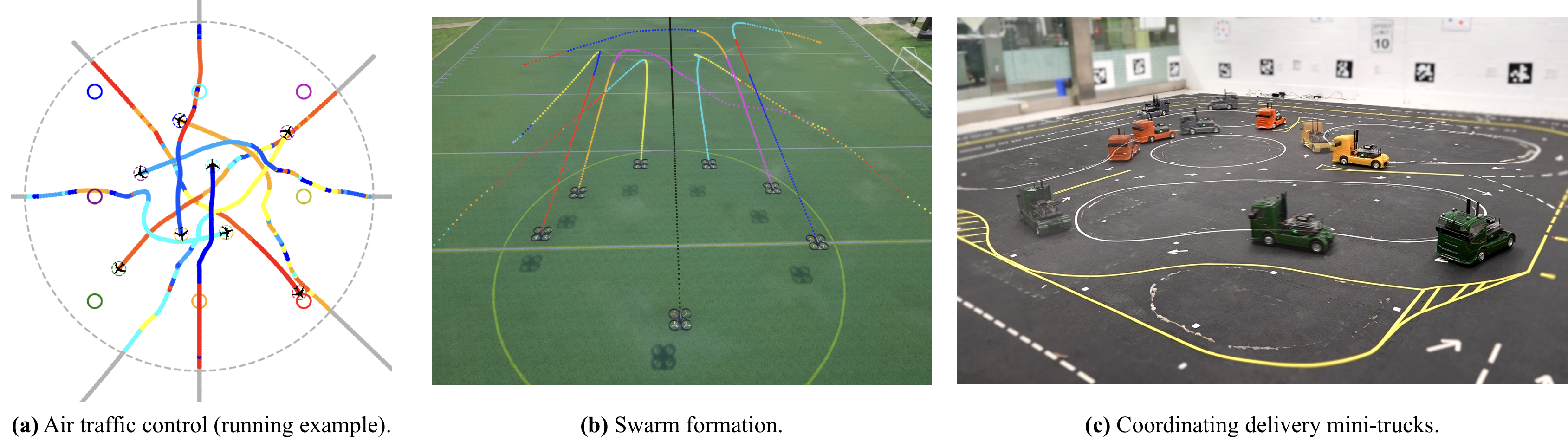}
\setcounter{figure}{0}
\vspace{-0.2in}
\captionof{figure}{
Our game-theoretic planning method computes the socially optimal Stackelberg equilibrium in real time.
\textbf{(a)}
Starting from an air traffic control zone with eight airplanes flying on collision courses, our method computes collision-free and socially optimal trajectories (warmer color denotes higher priority) compared to the baselines.
\textbf{(b)}
Our method handles moving targets when applied for a quadrotor swarm formation task in the AirSim simulator~\cite{shah2018airsim}.
\textbf{(c)}
Our method coordinates a delivery vehicle fleet in a scaled metropolitan area (vehicle snapshots corresponding to later time steps have higher transparency).
The video is available at~\url{https://youtu.be/wb6cMYJ43-s}
}
\label{fig:front}
\vspace{-0.2in}
\bigskip}

\makeatother
\maketitle
\begingroup\renewcommand\thefootnote{\textsection}
\footnotetext{H. Hu and G. Dragotto contributed equally.}
\endgroup

\begin{abstract}
We consider the multi-agent spatial navigation problem of computing the socially optimal order of play, \ie, the sequence in which the agents commit to their decisions, and its associated equilibrium in an $N$-player Stackelberg trajectory game. We model this problem as a mixed-integer optimization problem over the space of all possible Stackelberg games associated with the order of play's permutations.
To solve the problem, we introduce Branch and Play (B\&P), an efficient and exact algorithm that provably converges to a socially optimal order of play and its Stackelberg equilibrium.
As a subroutine for B\&P, we employ and extend sequential trajectory planning, \ie, a popular multi-agent control approach, to scalably compute valid local Stackelberg equilibria for any given order of play.
We demonstrate the practical utility of B\&P to coordinate air traffic control, swarm formation, and delivery vehicle fleets.
We find that B\&P consistently outperforms various baselines, and computes the socially optimal equilibrium.
\end{abstract}

\iftoggle{preprint}{}{
    \IEEEpeerreviewmaketitle
}

\section{Introduction}
\label{sec:intro}
On Monday morning, three aircraft are flying on a collision course. How should the air traffic controller redirect the airplanes to avoid collisions while minimizing their eventual delays? Meanwhile, a fleet of automated heavy-duty and slow-moving delivery trucks are finishing their tasks across a congested road network. Which route should the trucks be allowed to take so that it minimizes their impact on traffic while also minimizing their fuel consumption? 

The previous examples share three critical attributes. First, the outcome of the interaction is contingent on \textit{the order of play}, \ie, the sequence in which the agents commit to their decisions. Second, the systems are safety-critical and are subject to the oversight of an external \textit{regulator}. Third, the agents are self-interested and may have conflicting preferences and different information; these factors limit their willingness to completely delegate the decision-making authority to the regulator.
These three attributes are particularly evident in scenarios such as \gls{ATC}, delivery logistics, and groups of autonomous robots operating in close proximity. In these scenarios, the order of play crucially affects the operational effectiveness of the group as a whole.
Therefore, a key challenge for the regulator is determining an optimal order of play that is socially optimal, \ie, that maximizes the sum of the agents' utilities. Naturally, as the number of agents increases, the problem becomes combinatorially complex and computationally intractable.  This complexity has tangible repercussions in domains where safety and reliability are non-negotiable, and where a high-level regulator has oversight responsibilities over the collective outcome and partial authority over the participating agents.

Non-cooperative game theory provides principled frameworks to model safety-critical systems where several self-interested agents interact.
Existing approaches
often fail to scale to many agents, and are capable of addressing bilevel or trilevel scenarios at most,
rendering them impractical for most applications involving more than \new{three} agents%
~\cite{fisac2019hierarchical, sadigh2018planning, zhao2022stackelberg, tian2021safety}.
Other methods often resort to heuristic strategies to gain computational tractability. For instance, the “first come first served” rule~\cite{chen2015safe} 
proposed for air traffic control
offers a
solution that, while effective in certain contexts, lacks the flexibility and optimality necessary for broader applications. 
Those heuristic-based solutions ultimately compromise the potential to achieve the socially optimal outcome. Importantly, unlike the existing approaches, we do not assume that the order of play is given.

\p{Contributions} To overcome these limitations, we propose \gls{BNP}, an efficient and exact branch-and-bound method to compute a socially optimal order of play. Specifically,  we focus on sequential trajectory planning tasks formulated as $\nagents$-player trajectory Stackelberg games, \ie, sequential games where $\nagents$ sequentially plan their trajectories.
\gls{BNP} is an iterative method that implicitly explores the search space with the hope of avoiding the costly enumeration of all the possible orders of play. 
We contribute to both the science and the systems aspects of socially-aware game-theoretic planning as follows:
\begin{enumerate}[(a.)]
    \item \textbf{Science.} We formulate the problem of finding the optimal order of play as a mixed-integer optimization problem. Algorithmically, we propose~\gls{BNP}, a novel method that implicitly enumerates the search space over possible leader--follower orderings and is guaranteed to find the socially optimal Stackelberg equilibrium for general non-cooperative trajectory games. 
    In games with aligned interaction preferences (\eg, collision avoidance), we prove for the first time that \gls{STP}, a popular multi-agent control approach, produces a \gls{LSE} in a single pass.
    \item \textbf{Systems.}  We deploy a real-time \gls{BNP} for simulated air traffic control and quadrotor swarm formation, and we perform hardware experiments for the coordination of a delivery vehicle fleet in a scaled metropolis.
    We evaluate our solution against several baseline approaches (\eg, first-come-first-served and Nash solutions) and we investigate their quantitative and qualitative differences. Our computational tests show that \gls{BNP} consistently outperforms the baselines and finds socially optimal equilibria.
\end{enumerate}

\p{Overview}
We organize this paper as follows.
\autoref{sec:related_work} discusses the related works, while 
\autoref{sec:formulation} introduces the problem formulation. 
\autoref{sec:wpf} presents our key contribution, the \gls{BNP} algorithm, and its properties.
\autoref{sec:stp} discusses how to compute a Stackelberg equilibrium using \gls{STP}. ~\autoref{sec:result} presents our experiments, and ~\autoref{sec:conclusion} and ~\autoref{sec:limitation} discuss the limitations, future directions and conclusions.

\begin{figure*}[!hbtp]
    \centering
    \includegraphics[width=\iftoggle{preprint}{\textwidth}{1.6\columnwidth}]{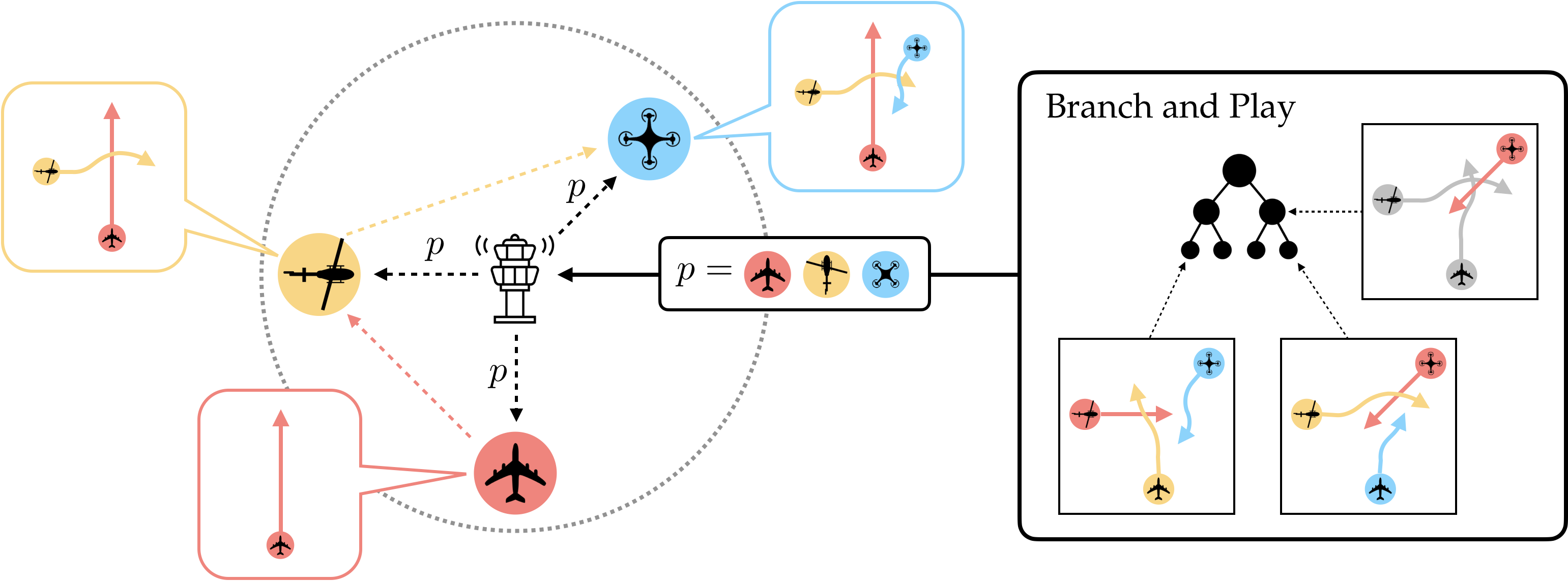}
    \caption{An overview of \gls{BNP} applied to air traffic control, where we employ \gls{STP} as subgame solver. \gls{BNP} computes the socially optimal order of play (the jet is the leader, followed by the helicopter, and finally the quadrotor) and broadcasts it to all airborne agents in the \gls{ATC} zone.
    Once the regulator broadcasts the order of play, \gls{STP} conditioned on the socially optimal order of play runs at a higher frequency: Each agent optimizes its own trajectory based on predecessors' plan, and communicates this information with its successors.}
    \label{fig:framework}
\end{figure*}

\vspace{0.4cm}
\section{Related Work}
\label{sec:related_work}

\subsection{Game-Theoretic Planning}
Non-cooperative (dynamic) trajectory games%
% (or dynamic trajectory games)
~\cite{bacsar1998dynamic} are a powerful tool to model multi-agent coordination tasks, for instance, autonomous driving~\cite{sadigh2018planning,zanardi2021urban,hu2023active}, manipulation~\cite{zhao2022stackelberg}, multi-robot navigation~\cite{zhao2023stackelberg}, physical human-robot interaction~\cite{li2019differential}, power networks~\cite{hu2020non}, and space missions~\cite{palafox2023learning}. 
Nash~\cite{nash1951non} and Stackelberg~\cite{stackelberg1934} equilibria
are popular solution concepts for such games.
However, their computation often poses several practical challenges.

\p{Nash games} Computing Nash equilibria is often challenging due to the inherently non-convex nature of many games, \eg, games with general dynamics and cost functions. However, in some specific cases, as in the \gls{LQ} setting, there are efficient analytical methods to compute Nash equilibria ~\cite[Ch.6]{bacsar1998dynamic}.
Fridovich et al.~\cite{fridovich2020efficient} leverage this idea by building an iterative \gls{LQ} approximation scheme to compute approximate Nash equilibria of unconstrained, non-convex trajectory games.
Laine et al.~\cite{laine2023computation} provide necessary and sufficient conditions for the existence of generalized Nash equilibria accounting for hard constraints, and propose an algorithm to compute approximate equilibria.
Recently, Nash games are also studied in stochastic settings~\cite{so2023mpogames,mehr2023maximum,hu2023deception,lidard2024blending}.
However, Nash games cannot efficiently model scenarios where decisions are sequential and the information available to the players is asymmetric and time-dependent.

\p{Stackelberg games}
Stackelberg games are sequential games played in rounds. Their solutions, Stackelberg equilibria, are often the solution of a hierarchical (\ie, nested) optimization problem, and unlike their Nash counterpart, they are explicitly parameterized by the (given) order of play.
This feature makes Stackelberg games suitable for modeling a wide range of multi-agent interaction scenarios where a leader (\ie, the first mover) plays before a set of followers. For instance, which vehicle should stop and yield to others at a traffic intersection; which aircraft should change its altitude when multiple aircraft are flying on a collision course; or which robot should make the first move when collaborating with other autonomous agents.
Zhao et al.~\cite{zhao2022stackelberg} model collaborative manipulation as a Stackelberg trajectory game and propose an optimization-based approach to find a Stackelberg equilibrium.
Tian et al.~\cite{tian2021safety} employ Bayesian inference to estimate who is the leader (or follower) in a Stackelberg trajectory game.
Khan and Fridovich-Keil~\cite{khan2023leadership} exploit a similar idea to solve Stackelberg trajectory games online with \gls{LQ} approximations. 
Previous works, however, are mostly limited to 2-player setting and assume that the order of play is fixed. In addition, an external regulator often oversees the players' actions, \eg, as in air traffic control, autonomous taxis, and warehouse robots. In this setting, the order of play becomes a natural ``tuning knob'' the regulator optimizes to improve the social outcome, \eg, by reducing traffic congestion.
Unfortunately, the presence of the regulator adds a new level of difficulty to the problem, as the number of possible permutations of the order of play scales factorially with the number of players. In this paper, we precisely address the setting where the regulator has to determine the socially optimal order of play inducing a Stackelberg equilibrium of an $N$-robot ($N > 2$) Stackelberg trajectory game.

\subsection{Multilevel Optimization for Stackelberg Games}
In optimization, Stackelberg games are often solved via bilevel (or, in general, multilevel) optimization techniques \cite{kleinert2021survey,beck2021gentle,dempe2020bilevel}.
A bilevel problem is an optimization problem whose constraint includes another optimization problem, \ie, a nested lower-level problem. These bilevel formulations are extremely useful in practical contexts such as pricing~\cite{brotcorne2008joint}, regulation of complex energy markets~\cite{wnms}, the protection of critical infrastructure \cite{baggio2021multilevel,dcrit2023}, and collaborative manipulation~\cite{zhao2022stackelberg}. However, the complexity of computing a solution to the associated multi-level optimization problem rises up by one layer in the polynomial hierarchy of computational complexity for each round of decisions \cite{jeroslow1985polynomial}. For instance, determining the Stackelberg equilibrium where $2$ players solve a linear optimization problem is already $\mathcal{NP}$-hard. In this paper, we employ and extend \gls{STP} to efficiently compute, for any given order of play, an \gls{LSE} of an $\nagents$-player Stackelberg trajectory game with $\nagents>2$.

\subsection{Multi-Robot Trajectory Planning}
\new{Our work is closely related to both established literature and recent advances in multi-robot trajectory planning methods.}

\p{Multi-agent pathfinding}
\new{\gls{MAPF} aims at finding collision-free and goal-reaching trajectories for a group of robots.
Unlike game-theoretic planning methods that optimize robot states and actions in continuous spaces, most \gls{MAPF} approaches find collision-free paths on a grid or a graph with discretized states or actions.
Early work~\cite{hopcroft1984complexity} shows that centralized \gls{MAPF} on a grid is PSPACE-hard, hence generally computationally intractable.
In contrast to these methods, decoupled or distributed approaches break down the task into independent or loosely-interdependent problems for each robot. Local Repair A*~\cite{zelinsky1992mobile} extends the classical single-agent A* algorithm~\cite{hart1968formal} to the \gls{MAPF} setting by iteratively rerouting agent paths to resolve conflicts but is prone to computation bottlenecks such as cycles.
Silver et al.~\cite{silver2005cooperative} designed a class of Cooperative A* algorithms, which mitigate the cycling issue by accounting for other robot plans in each agent's pathfinding.
Other open challenges and variants of \gls{MAPF} problems are extensively summarized in~\cite{stern2019multi,salzman2020research}.
As pointed out in~\cite{silver2005cooperative}, a critical factor influencing the performance of cooperative pathfinding algorithms is the agents' ordering, which the authors suggest should be dynamically adjusted.
Our proposed \gls{BNP} algorithm is specifically designed to compute the (time-varying) ordering, a key factor in optimizing group efficiency for multi-robot trajectory planning tasks. This novel approach can be seamlessly integrated with existing \gls{MAPF} methods, promising a significant boost in their performance.
Moreover, our framework considers general non-cooperative trajectory games, and it is capable of modeling interactions beyond typical \gls{MAPF} formulations.
}

\p{Prioritized planning}
\new{
\gls{STP} shares conceptual similarities with prioritized planning, an efficient class of \gls{MAPF} algorithms (see, \eg, ~\cite{erdmann1987multiple,latombe1991robot,fujimura1991motion}).
In prioritized planning, each robot has a unique priority, and the algorithm advances sequentially from the robot with the highest priority to the one with the lowest priority, with the rule that each robot avoids its higher-priority peers.
Čáp et al.~\cite{vcap2015prioritized} studied prioritized planning assuming a 2D workspace and showed that the classical prioritized planning approach~\cite{erdmann1987multiple} is incomplete, \ie, it cannot guarantee a conflict-free solution even if one exists.
They developed a revised version of prioritized planning and proved that it guarantees completeness for a subset of planar pathfinding problems.
In general, \gls{MAPF} problems solvable with prioritized planning can be viewed as a special case of a Stackelberg trajectory game, where the only coupling between each pair of agents is avoiding collisions.
While ensuring completeness at the trajectory level is beyond the scope of this paper, our \gls{STP}-based Stackelberg game solver may be integrated with prioritized planning approaches such as~\cite{vcap2015prioritized} to guarantee recursive feasibility in planar \gls{MAPF} problems.
}

\begin{figure*}[!hbtp]
    \centering
    \includegraphics[width=\iftoggle{preprint}{\textwidth}{1.7\columnwidth}]{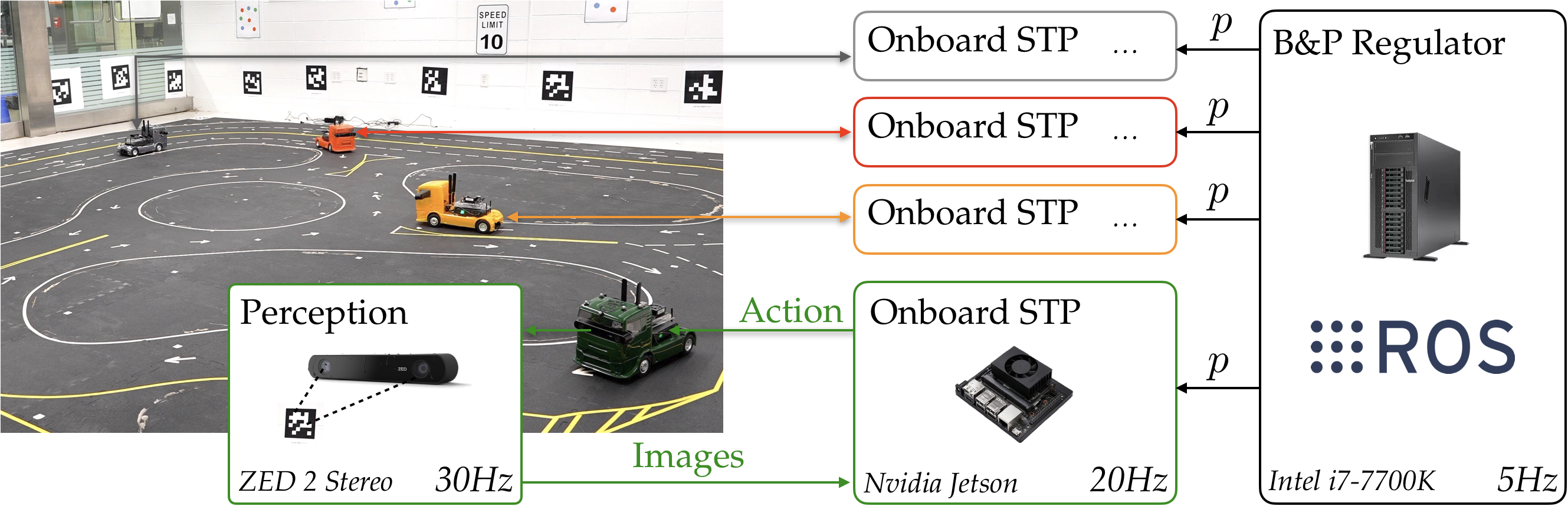}
    \caption{A ROS-based implementation of \gls{BNP} for coordinating a delivery vehicle fleet in a scaled metropolis. Each truck performs onboard computation with an Nvidia Jetson Xavier NX computer, which runs a visual-inertial SLAM algorithm for localization. The computer also solves the \gls{STP} for the ego vehicle's actions (acceleration and steering angle) based on leading vehicles' planned trajectories communicated wirelessly. The \gls{BNP} is solved on a desktop and the optimal permutation $\perm$ is then broadcast to each truck.}
    \label{fig:system}
\end{figure*}

\p{Formation control}
\new{A challenging subclass of \gls{MAPF} problems is formation control, which concerns navigating a group of robots, such as quadrotors or vehicles, in a desired geometric configuration while avoiding collisions.
Early work predominantly focuses on formation maintenance and stabilization using control-theoretic approaches, \eg, Lyapunov methods~\cite{ogren2001control}, distributed \gls{MPC}~\cite{dunbar2006distributed}, and distance-based flocking control~\cite{dimarogonas2008stability}.
Alonso-Mora et al.~\cite{alonso2017multi} achieve collision-free formation control in dynamic environments by adjusting formation parameters, such as size and orientation, within a constrained optimization framework.
Zhou et al.~\cite{zhou2022swarm} leverages spatial-temporal trajectory planning to enable formation navigation of drone swarms in the wild.
In this paper, we also apply \gls{BNP} to a quadrotor formation and spatial navigation task and demonstrate its superior time efficiency compared to a baseline method that plans with a fixed ordering.
Note that our approach applies to collaborative tasks, such as formation control and cooperative pathfinding, while addressing, more generally, \textit{non-cooperative} game-theoretic interactions involving self-interested agents.
}

\section{Problem Formulation}
\label{sec:formulation}

\p{Notation} Following standard game theory notation, let $(\cdot{})^{-i}$ be $(\cdot{})$ except for its $i$-th element.
Similarly, we use the subscript $t$ to refer to the time $t \in [T] := \{0,\dots, T-1\}$. For each player $i \in \orderset := \{1,\dots,N\}$, we let $\preset^i := \{1,\ldots,i-1\}$ and $\sucset^i := \{i+1,\ldots,N\}$ be the predecessor and successor set of $i$, respectively.
Let $\permset$ be the set containing all permutations of the order of play for $N$ players.

\p{System dynamics}
We consider an $N$-player $N$-hierarchy discrete-time (sequential) trajectory game governed by the nonlinear system
\begin{equation}
\label{eq:dyn_sys}
\state_{t+1} = \dyn_t(\state_t, \ctrl_t),
\end{equation}
where $\state = (\state^1,\ldots,\state^N) \in \reals^{n}$ and $\ctrl=(\ctrl^1,\dots,\ctrl^N)$ are the system's state and controls, and $\ctrl^i \in \cset^i \subseteq \reals^{m_i}$.

\p{Objectives and strategies}
We model each player $i$ as a rational decision maker, who minimizes a cost function
\begin{equation}
\label{eq:game}
\cost^i(\policy) = \textstyle\sum_{k=0}^{T} \stagecost^i_{k}(\state_k, \ctrl_k^{i}),
\end{equation}
where $\stagecost^i_{k}(\cdot)$ is player $i$'s stage cost function at time $k$.
The control action $\ctrl_k^{i}$ at time $k$ is given by  a \emph{strategy} $\policy^i: \reals^{n} \times [0,T] \rightarrow \cset^i$.
The class of available strategies $\policy^i\in\pset^i$ can encode
different \emph{information structures}, \eg, open-loop if $\ctrl^i_t = \policy^i_t(x_0)$ or feedback if $\ctrl^i_t = \policy^i_t(x_t)$. Given a strategy profile $\policy$, its social cost is $\costsocial(\policy):=\sum_{i \in \orderset}  \cost^i(\policy)$

\p{Cost separation by interaction}
In this paper, we focus on a class of trajectory games where \emph{interactions} solely involve collision-avoidance specifications. This class of game captures a wide range of non-cooperative multi-robot motion planning tasks.
Specifically, we assume each player's stage cost has an additive structure
\begin{equation}
    \label{eq:cost_assump}
    \stagecost^i_t(\state_t, \ctrl^i_t) = \stagecostper^i_t(\state^i_t, \ctrl^i_t) + \stagecostsafe_t^{i}(\state_t),
\end{equation}
where $\stagecostper^i_t(\state^i_t, \ctrl^i_t)$ is the \emph{individual} cost depending only on player $i$'s state and control, and $\stagecostsafe_t^{i}(\state_t) := \textstyle\sum_{j \in \{-i\}} \stagecostsafe^{ij}_t(\state^i_t, \state^j_t)$ is the \emph{interactive safety} cost penalizing collisions between players $i$ and all other players $j \in \{-i\}$. We also assume the pairwise safety cost is \emph{symmetric}, \ie,  $\stagecostsafe^{ij}_t(\state^i_t, \state^j_t) = \stagecostsafe^{ji}_t(\state^j_t, \state^i_t)$.
Therefore, we can let $\cost^i(\policy) = \costindi^i(\policy) + \costsafe^i(\policy)$ where $\costindi^i(\policy) := \sum_{k=0}^{T} \stagecostper^i_k(\state^i_k, \ctrl^i_k)$ and $\costsafe^i(\policy) := \sum_{j \in \{-i\}} \sum_{k=0}^{T} \stagecostsafe^{ij}_k(\state^i_k, \state^j_k)$.
In this paper, we consider \emph{Stackelberg equilibria} as our solution concept, which is formalized in \autoref{def:GSE}.

\begin{definition}[Global Stackelberg equilibrium]
\label{def:GSE}
\new{The strategy profile $\policy^* := (\policy^{1,*},\ldots,\policy^{N,*})$ is the \gls{GSE} of trajectory game \eqref{eq:game} if for each player $i$, the associated cost with all predecessors and $i$ playing the \gls{GSE} is no worse than the cost obtained when all predecessors playing the \gls{GSE} and $i$ playing any other strategy than $\policy^{i,*}$.
Formally, $\policy^*$ is the \gls{GSE} if for all $i \in \orderset$,
\begin{equation}
\begin{aligned}
\label{eq:GSE}
    \sup_{\policy^{\succ i} \in \resmap^{\succ i}(\policy^{\preceq i,*})} &\cost^i({\policy}^{\preceq i,*}, \policy^{\succ i}) \leq \\
    &\sup_{\tpolicy^{\succ i} \in \resmap^{\succ i}({\policy}^{\prec i,*}, \tpolicy^{i})} \cost^i({\policy}^{\prec i,*}, \tpolicy^{i}, \tpolicy^{\succ i}),
\end{aligned}
\end{equation}
for all $\tpolicy^i \in \pset^i$, where $\resmap^{\succ i}: \policyset^{\prec i} \times \policyset^i \rightrightarrows \policyset^{\succ i}$ is the set-valued \emph{optimal response map} of player $i$. Specifically, $\resmap^{\succ i}$ can be recursively defined from the last player as
\begin{equation}
\begin{aligned}
\label{eq:GSE:opt_res}
    \resmap^{\succ i} (\policy^{\preceq i}) &:= \bigcup_{\policy^{j} \in \resmap^{j}(\policy^{\preceq i})} \{\policy^{j}\} \times \resmap^{\succ j}(\policy^{\preceq j}),\\
    \resmap^{j}(\policy^{\preceq i}) &:= \left\{\policy^{j} \in \tpolicyset^{j} \mid \forall \tpolicy^{j} \in \tpolicyset^{j}: \right.\\
    \sup_{\policy^{\succ j} \in \resmap^{\succ j}} &\cost^{j}(\policy^{\prec j}, \policy^{j}, \policy^{\succ j}) \leq \\
    &\left.\sup_{\tpolicy^{\succ j} \in \resmap^{\succ j}} \cost^{j}(\policy^{\prec j}, \tpolicy^{j}, \tpolicy^{\succ j}) \right\}.
\end{aligned}
\end{equation}
where $j$ denotes the next player after $i$ in the order of play.
In \eqref{eq:GSE:opt_res}, the $\sup$ operation is used to guarantee that the \gls{GSE} is unique by enforcing that each player assumes the worst-case realization of its followers' choices within their joint optimal response set.}
\end{definition}

\new{In practice, computing the \gls{GSE} of a general non-cooperative trajectory game is intractable for all but the simplest cases.
Therefore, in the same spirit of the recent efforts on solving Stackelberg games in real-time~\cite{zhao2022stackelberg,khan2023leadership,li2024computation}, we aim at providing a best-effort approximation of the \gls{GSE} through a \textit{local} Stackelberg equilibrium.}

\begin{definition}[Local Stackelberg equilibrium]
\label{def:LSE}
\new{The strategy profile $\policy^* := ({\policy}^{1,*},\ldots,{\policy}^{N,*})$ is an \gls{LSE} of \eqref{eq:game} if there exists an open neighbourhood $\bpolicyset({\policy^*}) \subseteq \policyset$ such that $\policy^* \in \bpolicyset({\policy^*})$ and for each player $i \in \orderset$, \eqref{eq:GSE} holds for all $\tpolicy^i \in \bpolicyset^i({\policy^*})$.}
\end{definition}

\begin{remark}
    \new{In general, the solution quality (in terms of the social cost) of an \gls{LSE} depends on two factors: the players' order of play and the Stackelberg game solver. In this paper, we focus on computing the socially optimal order of play using \gls{BNP}, subject to the local optimality associated with an \gls{STP}-based game solver.}
\end{remark}

As per \autoref{def:LSE}, the players' order of play intrinsically parameterizes an~\gls{LSE}. Therefore, the key question we would like to answer is: \emph{How can we efficiently find the socially optimal order of play and its associated \gls{LSE} strategy?}
In \autoref{def:order}, we formalize the concept of permutation and its one-to-one correspondence to an order of play.

\begin{definition}[Permutation and order of play]
    A \emph{permutation} is a tuple $\perm := (\perm_1, \perm_2, \ldots, \perm_N) \in \permset$ with non-repeating elements $\perm_i \in \orderset$, each indicating player $i$'s \emph{order of play} in the Stackelberg game, e.g., $\perm = (3,1,2)$ means player 3 is the leader, followed by player 1, and player 2 plays the last.
    \label{def:order}
\end{definition}

Since $\policy$ is always parametrized by a permutation $\perm$, we will employ the notation $\policy(\perm)$ to refer to the strategy $\policy$ given $\perm$.
Furthermore, in \autoref{def:optimal}, we define the (an) optimal permutation and its associated permutation.
\begin{definition}[socially optimal equilibrium]
    A strategy profile $\policy(\perm)$ with permutation $\perm$ is socially optimal if and only if $\perm \in \argmin_{\tilde{\perm} \in \permset} \sum_{i \in \orderset} \cost^i(\policy(\tilde{\perm}))$.
    \label{def:optimal}
\end{definition}

\begin{example}
We illustrate our approach with an air traffic control (ATC) example involving eight airplanes initially flying on collision courses (\autoref{fig:front} (a)).
The dynamic model of each airplane can be found in Appendix~\ref{apdx:model}.
The \gls{ATC} zone is a circular region centered at $(0,0)$ with a radius of 2.5 (the unit is abstract).
Outside the \gls{ATC} zone, an aircraft ignores other agents and uses an \gls{ILQR} planner~\cite{li2004iterative} to reach its target.
The individual cost $\stagecostper^i_t(\state^i_t, \ctrl^i_t)$ encodes tracking the target state and regularizing control inputs, and safety cost $\stagecostsafe^{ij}_t(\state^i_t, \state^j_t)$ penalizes collisions. We say that two airplanes $i$ and $j$ are in collision at time $t$ if $\|p^i_t - p^j_t\| \leq 0.2$.
A commonly adopted \gls{ATC} strategy is \gls{FCFS}~\cite{chen2015safe}, which assigns a higher priority to an airplane that enters the zone earlier.
We will show that the dynamic assignments of the order of play of \gls{BNP} improve the airplane's coordination and yield a better social outcome.
\end{example}

\p{Mixed-integer formulation}
We formulate the problem of finding a \new{socially optimal Stackelberg equilibrium} as the 
\new{mixed-integer optimization problem}

\begin{subequations}
\label{eq:mio}
\begin{align}
    \min_{\{\intvar^\perm, \policy (\perm)\}_{\perm \in \permset} }~&\costsocial(\policy) := \sum_{i \in \orderset} \cost^i(\policy)  \label{eq:mio:objective}\\
    \st \quad\;\;\;&\state_{t+1} = \dyn_t(\state_t, \policy_{t}(\state_t)),  &&\forall t \in [T], \label{eq:mio:dynamics}\\
    &\policy_t = \sum_{\perm \in \permset} \intvar^\perm \bar{\policy}_t(\perm), &&\forall t \in [T],  \label{eq:mio:equality}\\
    & \bar{\policy}(\perm) \in \eqm(\perm), &&\forall \perm \in \permset, \label{eq:mio:equilibria}\\
    &\intvar^\perm \in \{0,1\},~\quad \sum_{\perm \in \permset} \intvar^\perm = 1,  &&\forall \perm \in \permset.  \label{eq:mio:binary}
\end{align}
\end{subequations}
For each permutation $\perm \in \permset$, we introduce a binary variable $\intvar^\perm$ that equals $1$ if and only if $\perm$ is the optimal permutation, and the set of variables $\policy(\perm)$ associated with the strategy profile resulting from $\perm$. In \cref{eq:mio:dynamics}, we enforce the system's dynamics, while in \cref{eq:mio:equality} we define the optimal policy $\policy$ as a $0-1$ combination of $\policy(\perm)$ for each $\perm$. In \cref{eq:mio:equilibria}, we require $\policy(\perm)$ to be in 
where $\eqm(\perm)$, \ie, the set of \gls{LSE} policies under $\perm$. In \cref{eq:mio:binary}, we require that only one permutation $\perm$ is active via the $\intvar^\perm$ variable.
Finally, without loss of generality, in \cref{eq:mio:objective}, we minimize the sum of players' costs.
Naturally, solving \cref{eq:mio} is at least $\mathcal{NP}$-hard even if we assume $\eqm(\perm)$ is given for any $\perm$.

\begin{remark}
    \new{Mixed-integer optimization problem}~\eqref{eq:mio} supports different objectives $\costsocial$ encoding other specifications than socially optimal outcomes, \eg, a weighted sum of players' costs $\costsocial(\policy) := \sum_{i \in \orderset} \alpha_i \cost^i(\policy)$, $\alpha_i > 0$, or the worst-case cost $\costsocial(\policy) := \max_{i \in \orderset} \cost^i(\policy)$.
    Furthermore, our algorithm minimally requires that $\costsocial$ can be evaluated efficiently given $\policy$ and $\state_0$. We do not assume continuity and differentiability of $\costsocial$.
\end{remark}

%%%%%%%%%%%%%%%%%%%%%%%%%%%%%%%%%%%%%%%%%%%%%%%%%%%%%%%%%%%%%%%%
\section{Branch and Play}
\label{sec:wpf}

Although \cref{eq:mio} provides an explicit and exact formulation to determine a \new{socially optimal Stackelberg equilibrium}, its solution poses considerable practical challenges.
First, for any given $\perm \in \permset$, we need to efficiently and scalably compute an $\nagents$-player \gls{LSE} $\policy(\perm)$ for the associated trajectory game. The latter often includes nonlinear system dynamics~\eqref{eq:dyn_sys} and nonconvex cost functions~\eqref{eq:game}. Therefore, even for a single $\perm$, determining an \gls{LSE} often requires well-crafted computational methods.
Second, suppose we have access to an oracle (\eg, a game solver) to efficiently compute the equilibria $\policy(\perm) \in \eqm(\perm)$ under any $\perm$. 
As there are $\nagents!$ permutations of the players, there are factorially-many constraints \cref{eq:mio:equilibria} associated with each player. Therefore, the explicit enumeration of all the orders of play scales factorially with the number of players. 

Because of these practical challenges, we introduce an implicit enumeration scheme to possibly avoid the expensive computation of $\nagents!$ \glspl{LSE} associated with each permutation of the order of play. 
Specifically, we propose \gls{BNP}, a novel algorithm to efficiently determine a \new{socially optimal Stackelberg equilibrium} and its associated optimal order of play. In practice, \gls{BNP} implicitly enumerates the search space of  \cref{eq:mio} via a branch-and-bound \citep{land_automatic_1960} tree and a hierarchy of bounds associated with each \gls{LSE}. Each node of the \gls{BNP} search tree is associated with a (partial) permutation of the order of play, and a corresponding trajectory planning problem. To solve the latter, we employ \gls{STP}~\cite{chen2015safe,chen2018robust}, a distributed control framework for multi-agent trajectory optimization. Given a permutation $\perm$ of the order of play, we also prove, for the first time, that \gls{STP} yields an \gls{LSE} of the trajectory game~\eqref{eq:game}.
Without loss of generality, in this section, we assume access to a subgame solver that returns an \gls{LSE} associated with $\perm \in \permset$ and the corresponding Stackelberg game.
We illustrate our framework in~\autoref{fig:framework} (\gls{ATC}) and~\autoref{fig:system} (truck coordination).

\subsection{Incomplete Permutations and Their Bounds.}
\label{sec:wpf:incomplete_perm}
The implicit enumeration of \gls{BNP} requires us to enforce a hierarchy of bounds (or values) for each permutation $\perm \in \permset$. This is equivalent to computing the cost associated with an \gls{LSE} $\policy(\perm)$ associated with $\perm$.
We therefore define the concept of the value of a permutation.
% in \autoref{def:bound}.

\begin{definition}[Value of a permutation]
\label{def:bound}
The \emph{value} associated with a permutation $\perm$ is $\costsocial(\policy(\perm))$, where $\policy(\perm)$ is an \gls{LSE} of the Stackelberg game parametrized by $\perm$.
\end{definition}

However, as there are factorially many orders of play, bounding all the permutations is rather expensive. Therefore, with the hope of reducing the calls to the \gls{LSE} oracle, we devise the concept of \emph{incomplete permutations}, \ie, partial assignments of the orders of play.

\begin{definition}[Incomplete permutation]
An \emph{incomplete permutation} $\permincomplete = (\permincompleteassigned, \permincompleteunassigned)$ of order $\permorder$ is a permutation where $\permincompleteassigned$ contains the first $\permorder\le \nagents$ orders of play, and $\permincompleteunassigned$ contains the remaining unassigned player indices.
\label{def:incomplete}
\end{definition}

In practice, an incomplete permutation specifies the order of $\permorder$ players, while leaving the others unassigned. A complete permutation is, as in \autoref{def:incomplete}, an incomplete permutation of order \new{$\permorder=\nagents$}.
We refer to any permutation $\perm \in \permset$ as a \emph{complete permutation}, and to $\permsetincomplete$ as the set of all incomplete (and complete) permutations of any order such that $\permset \subseteq \permsetincomplete$.
\new{In this paper, we assume that, for an incomplete permutation $\permincomplete$, the unassigned players either (1). ignore all the interactive safety cost and play on their own, or (2). respect the interactive safety cost with respect to all assigned players, while ignoring all unassigned ones.
In both cases, the underlying trajectory game \eqref{eq:game} parameterized by $\permincomplete$ is well defined, and we denote an \gls{LSE} strategy of the game as $\policy(\permincomplete)$.
Therefore, the value of an incomplete permutation $\costsocial(\policy(\permincomplete))$ is also defined via \autoref{def:bound}; we will provide a detailed explanation of how to compute it using \gls{STP} in \autoref{sec:stp}.}

\begin{proposition}[Upper bound]
\label{prop:upper_bound}
For any complete permutation $\perm \in \permset$, its value is an upper bound on the optimal social cost, i.e., $\costsocial(\policy(\perm)) \ge \costsocial^*(\policy(\perm^*))$.
\end{proposition}

Let $\permmapping:\permsetincomplete \rightarrow \permsetincomplete$ be the mapping from any incomplete permutation of order $\permorder$ to the set of its descendant (complete or incomplete) permutations of order $\permorder+1$. For any incomplete permutation $\permincomplete$ of order $\permorder$, there are $\permorder!$ complete permutations descending from $\permincomplete$.
Therefore, if we can determine that there exists no improving \gls{LSE} in the descendants $\permmapping(\permincomplete)$, then we can avoid searching among those complete permutations. To achieve this, we should be able to ensure that the bound of $\permincomplete$ is a lower bound on any bound associated with its descendants $\permmapping(\permincomplete)$, as we formalize in \autoref{assump:admissibility}.

\begin{assumption}[Admissibility]
\label{assump:admissibility}
Let $\permincomplete \in \permsetincomplete$ be an incomplete permutation, and $\costsocial(\policy(\permincomplete))$ be its bound.
For the complete permutation $\perm^{\permincomplete} \in \permset$ descending from $\permincomplete$, $\costsocial(\policy(\perm^{\permincomplete})) \ge \costsocial(\policy(\permincomplete))$.
\end{assumption}

\autoref{assump:admissibility} enforces that the bound provided by an incomplete permutation never \emph{overestimates} the social cost of the complete permutation descending from it.
In practice, the assumption enables us to employ incomplete permutations as \emph{relaxations} of complete permutations, hence performing a one-to-many pruning of the search space via $\permmapping$.

\begin{remark}
    The admissibility assumption is common in algorithms of graph search and mixed-integer optimization.
    In A* search~\cite{hart1968formal}, admissibility means that a heuristic function, which estimates the cost from a node to the goal, never overestimates the actual minimal cost of reaching that goal.
    This assumption is crucial for proving that A* converges to the optimal path to the goal.
    In branch-and-bound ~\cite{land_automatic_1960} for mixed-integer optimization, admissibility is automatically satisfied since relaxing the integer requirements into continuous ones always returns an optimistic lower bound on the actual cost.
\end{remark}

\subsection{The Algorithm}
In \autoref{alg:bnp}, we present the pseudocode for \gls{BNP}. Given an instance of \cref{eq:game}, \gls{BNP} outputs an optimal permutation $\perm^*$ and its \gls{LSE} $\policy^*$. In Line~\ref{alg:cnp:initialize}, we initialize a list $\nodelist$ of nodes associated with the search tree, and the incumbent optimal cost, permutation, and controls. Without loss of generality, we initialize $\nodelist$ with a permutation $\permnode$ of order $\permorder=1$, \ie, we assign the order of one player.
We explore the search space until $\nodelist$ is empty by selecting a node $\node \in \nodelist$ (Line~\ref{alg:cnp:exploration}) via a so-called exploration strategy; we will describe a few exploration strategies later in this section. Let $\permnode$ be the permutation associated with $\node$. In Line~\ref{alg:cnp:bounding}, we call a subgame solver (\eg, \gls{STP} presented in  \autoref{sec:stp}) to compute the value of $\permnode$ and its strategy $\policy(\permnode)$. If the corresponding strategy profile $\policy(\permnode)$ is \textit{feasible} (\ie, it induces a feasible set of controls with no collisions or $\permnode \in \permset$), we then evaluate its bound. If it improves the incumbent best bound $\costsocial^*$, then we update the incumbent information. Afterward, we prune any open node in $\nodelist$ with non-improving bound values (Line~\ref{alg:cnp:pruning}). Otherwise, if $\node$ is not feasible, we branch by adding to $\nodelist$ the descendants $\permmapping(\permnode)$ of order $\permorder+1$, where $\permorder$ is the order of $\permnode$.

\begin{algorithm}[!ht]
\SetKwBlock{Repeat}{repeat}{}
\DontPrintSemicolon
\small
\caption{Branch and Play \label{alg:bnp}}
\KwData{An instance of \cref{eq:game} }
\KwResult{An optimal permutation $\perm^*$ and its \gls{LSE} $\policy^*$}
Initialize $\nodelist$ with the incomplete permutation of order $\permorder=1$, set $\costsocial^*\gets\infty$, $\perm^* \gets \emptyset$ and $\policy^* \gets \emptyset$\label{alg:cnp:initialize}\;
\While{ $\nodelist\neq \emptyset$}{
    {\textbf{\textcolor{Periwinkle}{Exploration Strategy}}}: select $\node \in \nodelist$ \label{alg:cnp:exploration}\;
    $\permorder \gets$ order of $\node$, and $\nodelist \gets \nodelist \backslash\{ \node\}$ \;
    {\textbf{\textcolor{Periwinkle}{Bounding}}}: compute $\costsocial(\policy(\permnode))$ and $\policy(\permnode)$ \label{alg:cnp:bounding}\;
    \If{$\costsocial(\policy(\permnode)) < \infty$ (\ie, the subproblem is feasible) }{
        \If{$\policy(\permnode)$ is feasible and $\costsocial^*>\costsocial(\policy(\permnode))$}{
            $\costsocial^*\gets \costsocial(\policy(\permnode))$, $p^*\gets \permnode$ and $\policy^* \gets \policy(\permnode)$\label{alg:cnp:complete}\;
            {\textbf{\textcolor{Periwinkle}{Pruning}}}: delete from $\nodelist$ any $\node$ such that $\costsocial(\policy(\permnode)) > \costsocial^*$ \label{alg:cnp:pruning}\;
        }
        \lElse{
            {\textbf{\textcolor{Periwinkle}{Branching}}}: Add to $\nodelist$ the descendants $\permmapping(\node)$ of $\permnode$ of order $\permorder+1$\label{alg:cnp:branching}
        }
    }
}
\textbf{return} $\perm^*$ and  $\policy^*$
\end{algorithm}

\subsection{Exploration, Pruning, and Convergence}
\label{sec:wpf:expl_pruning}
\gls{BNP} is a general-purpose algorithm, and it can be customized depending on each application's need. In this section, we briefly overview some customizations that expedite \gls{BNP} when seeking \gls{LSE} of large-scale trajectory games.

\subsubsection{Exploration strategies} In Line~\ref{alg:cnp:exploration} of \autoref{alg:bnp}, we select a node $\node$ from $\nodelist$ and perform the associated bounding operation. Naturally, as the tree is exponentially large, the exploration strategy significantly affects the convergence behavior. Traditional exploration strategies are depth-first or best-first. 
In depth-first, we explore the tree by prioritizing complete permutations over incomplete ones. The algorithm will find a feasible solution in at most $\nagents$ iterations, assuming the associated complete permutation is feasible. The feasible solution can help, in turn, to efficiently prune parts of the search tree due to its bound.
On the other hand, in best-first, we explore the tree by prioritizing permutations with the smallest bound possible. The algorithm will likely explore parts of the tree with small lower bounds with the hope of finding the optimal permutation.
Naturally, exploration strategies are heuristic. For instance, in mixed-integer optimization, solvers usually employ sophisticated exploration strategies~\cite{lodi2013heuristic}.

\begin{figure}[!hbtp]
    \centering
    \includegraphics[width=\iftoggle{preprint}{0.5\textwidth}{0.92\columnwidth}]{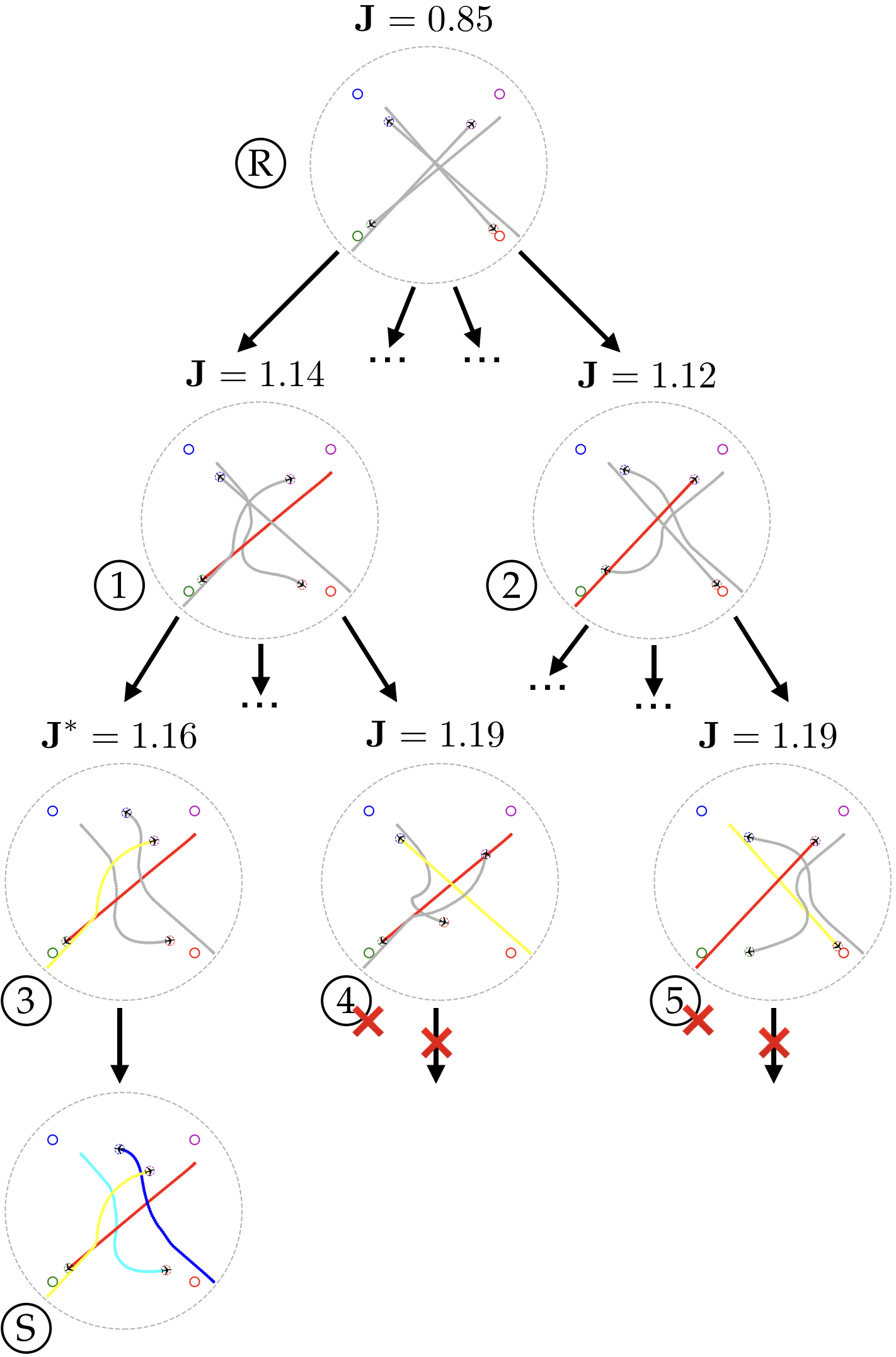}
    \caption{Illustration of a \gls{BNP} search tree. Trajectories of unassigned players, who avoid all assigned players, are plotted in grey. In the root node \circledtext{R}, all players are unaware of each other, resulting in a lower bound of the social cost. Players in nodes \circledtext{3}, \circledtext{4}, and \circledtext{5} are collision-free. Unassigned players in those nodes can be given any order of player and there is no need to descend further. Nodes \circledtext{4} and \circledtext{5} are pruned since they produce a higher cost than the feasible solution \circledtext{S}, which turns out to be the optimal solution.}
    \label{fig:bnp_tree}
\end{figure}

\subsubsection{Pruning strategies}
\label{subsec:pairwise_prune}
In addition to the native bound-based pruning \gls{BNP} performs (Line~\ref{alg:cnp:pruning} of \autoref{alg:bnp}), we exploit the structure of the underlying trajectory game to further prune the search tree. In particular, we introduce a new pruning strategy called \emph{pairwise-collision pruning}. Let $\node$ be a node associated with an incomplete permutation $\permnode = (\permnode_+, \permnode_-)$ of order $\permorder>1$, and $\policy(\permnode)$ be the associated strategy profile.
Consider a pair of distinct, unassigned players $i,~j \in \permnode_-$ and their trajectories associated with $\policy(\permnode)$.
\new{Recall from Section~\ref{sec:wpf:incomplete_perm} that all unassigned players in $\permnode_-$ will ignore the interactive safety cost with respect to each other.}
If trajectories $i$ and $j$ do not collide, permutations $(\permnode_+, i, j, \tilde{\perm}^\node_-)$ and $(\permnode_+, j, i, \tilde{\perm}^\node_-)$ are \textit{equivalent} for all $\tilde{\perm}^\node_-$ that assigns orderings to the rest of players in $\permnode_-$ other than $i$ and $j$.
Therefore, we can arbitrarily assign an order (\eg, $i$ plays before $j$) and prune the other children permutations in $\permmapping(\node)$.
Furthermore, if all players in $\permnode_-$ are collision-free, $\policy(\permnode)$ is automatically an \gls{LSE} strategy for games of all complete permutations descending from $\node$.

\begin{lemma}
\label{lem:bound_prune}
    Let \autoref{assump:admissibility} hold. Bound-based pruning does not exclude any node $\node$ containing the optimal solution of \cref{eq:mio}.
\end{lemma}

\begin{lemma}
\label{lem:col_prune}
    Pairwise-collision pruning does not exclude any node $\node$ containing the optimal solution of \cref{eq:mio}.
\end{lemma}

\subsubsection{Convergence} Theoretically, \gls{BNP} is guaranteed to find the \new{socially optimal Stackelberg equilibrium} for an instance of \cref{eq:game}, \ie, the  optimal solution of the \new{mixed-integer optimization problem~\eqref{eq:mio}, as long as \autoref{assump:admissibility} holds and the bounding is exact.}

\begin{theorem}
\label{thm:bnp}
    \gls{BNP}~returns the \new{socially optimal Stackelberg equilibrium}, \ie, the globally-optimal solution of~\eqref{eq:mio}, if \autoref{assump:admissibility} holds \new{and the bounding procedure returns the exact optimistic cost on any permutation $\tilde{\perm}$}. Furthermore, the worst-case number of iterations of \gls{BNP} is of order $O(\nagents!)$. 
\end{theorem}

\begin{proof}
    The convergence follows directly from admissibility (\autoref{assump:admissibility}) and validity of pruning (\autoref{lem:bound_prune} and \autoref{lem:col_prune}). 
\end{proof}

\begin{remark}
    Although the theoretical worst-case complexity of \gls{BNP} is factorial (as in a brute-force enumeration), the algorithm \new{rarely} hits the worst case (see \autoref{sec:result}).
    In practical applications (\eg, the \gls{ATC} in \autoref{fig:time_bnp}), the pruning strategies are effective at reducing the search space. Therefore, on average, Branch and Play converges in a fraction of the worst-case iterations \new{in the tested examples in \autoref{sec:result}}.
\end{remark}

\subsection{Receding Horizon Planning and Warmstart}
We often would like to apply \gls{BNP} in receding horizon fashion (\ie, \gls{MPC}). Namely, during each planning cycle, the robots execute only the first control action $\ctrl_0 = \policy_0(\state_0)$, the horizon shifts, and we run again \gls{BNP} with the updated initial state.
This procedure allows players to update their strategies \emph{and} switch the orders of play to account for dynamically changing environments and compensate for performance loss due to a finite (possibly short) planning horizon.

Inspired by existing \gls{MPC} algorithms~\cite{borrelli2017predictive}, we provide a warmstart procedure for receding horizon \gls{BNP}.
At planning step $t > 0$, before searching the tree from the root node, we work with the (partially) explored tree at the last planning step $t-1$ to obtain preliminary bounds on $\costsocial^*$.
First, to warmstart the subgame, we shift (by one time step) the equilibrated players' trajectories in each explored node.
Second, we solve the subgames for all explored leaf nodes (\ie, those with complete permutations) and their parent nodes to obtain lower and upper bounds on $\costsocial^*$, respectively.
Since in general the optimal permutations are likely to be similar (even the same) between two consecutive planning cycles, those bounds are expected to carry useful bound information.

%%%%%%%%%%%%%%%%%%%%%%%%%%%%%%%%%%%%%%%%%%%%%%%%%%%%%%%%%%%%%%%%
\section{Computing Stackelberg Equilibria with STP}
\label{sec:stp}

\gls{BNP} supports any off-the-shelf Stackelberg solver as a subroutine to perform the bounding. However, a computationally efficient and scalable game solver that \emph{guarantees} an \gls{LSE} still remains an open research challenge.
In this section, we extend the \gls{STP} algorithm originally developed in~\cite{chen2015safe,chen2018robust} for trajectory games~\eqref{eq:game} and provide a proof that it converges to an \gls{LSE}.
Throughout this section, we assume the players follow a given permutation $\perm \in \permset$.

\subsection{Sequential Trajectory Planning for Stackelberg Games}
Consider an agent $i \in \orderset$, an initial state $\state_0 \in \reals^{n}$, and all predecessors' planned trajectories $\btraj^{\prec i} := \bstate^{\prec i}_{[0:T]}$. \gls{STP} solves a \emph{single-agent} trajectory optimization problem with a \emph{surrogate} cost that ignores the safety costs of all successors, namely, 
\begin{equation}
\label{eq:surrogate_cost}
    \tcost^i(\policy^i(\state_0^i); \btraj^{\prec i}) = \costindi^i(\policy^i(\state_0^i)) + \tcostsafe^i(\traj^i; \btraj^{\prec i}),
\end{equation}
where $\tcostsafe^i(\traj^i; \btraj^{\prec i}) := \sum_{k=0}^{T} \sum_{j \in \preset^i} \stagecostsafe^{ij}_k(\state^i_k, \bstate^j_k)$ is the surrogate safety cost.
Other nonlinear trajectory game algorithms (\eg, ILQGame~\cite{fridovich2020efficient,khan2023leadership} and iterated best response~\cite{zanardi2021urban}) iteratively update players' strategies until convergence, which is not always guaranteed. In contrast, \gls{STP} computes the strategy of each player \emph{in a single pass}, leading to a deterministic computation overhead that scales linearly with the number of players (c.f.~\autoref{fig:time_stp}).

\begin{algorithm}[!htbp]
    \DontPrintSemicolon
    \small
	\caption{STP }
	\label{alg:stp}
	\KwData{Permutation $\perm$, current joint state $\state_t$, dynamics~\eqref{eq:dyn_sys}, stage costs~\eqref{eq:cost_assump}, horizon length $T>0$}
    \KwResult{Equilibrium trajectory $(\state_{[t:t+T]}, \ctrl_{[t:t+T]})$ and cost $\costsocial(\policy(\perm))$}
    \For{$i \in \perm$}{
        \label{alg:stp:solve_stp} $(\state^i_{[t:t+T]}, \ctrl^i_{[t:t+T]}) \gets$ \textsc{SolveSTP}$\left(\bstate^{\prec i}_{[t:t+T]})\right)$ \;
        Communicate $\state^i_{[t:t+T]}$ with player $i+1$\;
    }
\end{algorithm}

\subsection{Guaranteeing an \gls{LSE}}
\label{sec:guarantee_LSE}
In order to guarantee that \gls{STP} finds an \gls{LSE} for~\eqref{eq:game}, we introduce the Assumption \ref{assump:STP:best_f}.

\begin{assumption}[Best-effort cautious follower]
\label{assump:STP:best_f}
    Given an initial state $\state_0 \in \reals^{n}$, for all $i \in \orderset$, the policy $\policy^i(\state^i_0)$ satisfies, for an open neighbourhood $\tpolicyset^i(\policy^i) \ni \policy^i$,
    \begin{equation}
    \label{eq:assump:STP:best_f_1}
    \begin{aligned}
        \tcost^i(\policy^{i}(\state^i_0)) \leq \tcost^i(\tpolicy^{i}(\state^i_0)),~\quad~\forall\tpolicy^i \in \tpolicyset^i(\policy^i).
    \end{aligned}
    \end{equation}
    Moreover, the policy $\policy^i(\state^i_0)$ of each follower $i \in \orderset \setminus \{1\}$ satisfies, for an open neighbourhood $\tpolicyset^i(\policy^i) \ni \policy^i$,
    \begin{equation}
    \label{eq:assump:STP:best_f_2}
        \tcostsafe^i(\traj^i; \ttraj^{\prec i}(\tpolicy^1)) \leq \tcostsafe^i(\ttraj^i(\tpolicy^i); \ttraj^{\prec i}(\tpolicy^1)),
    \end{equation}
    for all $\tpolicy^i \in \tpolicyset^i(\policy^i)$ and predecessors' state trajectories $\ttraj^{\prec i}(\tpolicy^1)$ resulting from the leader's policy $\tpolicy^1 \in \tpolicyset^1(\policy^1)$.
\end{assumption}
\autoref{assump:STP:best_f} states that each player computes a locally-optimal policy with respect to their surrogate costs $\tcost^i$, and the followers additionally ensure that they cannot further decrease the safety cost \emph{despite} any local deviations of the leader policy.
In practice, this implies \autoref{lem:courtesy}.

\begin{lemma}[Neutralized courtesy]
\label{lem:courtesy} 
    Let \autoref{assump:STP:best_f} hold.
    The leader cannot improve by switching to another policy, \ie, there is no $\tpolicy^1 \in \tpolicyset^1(\policy^1)$ such that $\cost^1(\tpolicy^{1}(\state^1_0)) < \cost^1(\policy^{1}(\state^1_0))$.
\end{lemma}

\begin{proof}
    See Appendix~\ref{apdx:lem:courtesy}.
\end{proof}

\begin{theorem}[\gls{STP} finds an \gls{LSE}]
\label{thm:STP}
    If \autoref{assump:STP:best_f} hold, then the \gls{STP} strategy tuple $(\policy^{1},\ldots,\policy^{N})$ is an \gls{LSE} of game~\eqref{eq:game}.
\end{theorem}

\begin{proof}
    The result is a direct consequence of~\autoref{lem:courtesy}. All followers optimally respond to the leader with their \gls{STP} policies, and the leader has no incentive to deviate from $\policy^1$ to reduce its cost $\cost^1$.
\end{proof}

\p{Connection to lexicographic preferences}
We can employ \gls{STP} to compute Stackelberg equilibria with lexicographic preferences, \eg, urban driving games~\cite{zanardi2021urban}, where the \gls{LSE} in~\autoref{def:LSE} is instead defined with the lexicographic total order $\preceq$ on the tuple $\left\langle \costsafe^i, \costindi^i\right\rangle$.
In this case, \autoref{assump:STP:best_f} can be \emph{relaxed} and we only need to enforce~\eqref{eq:assump:STP:best_f_1} for the leader, assuming that the follower's \gls{STP} policy $\policy^i$ is the local unique minimum of $\tcostsafe^i$.
The reasoning of \autoref{lem:courtesy} then applies directly here: The leader cannot deviate to reduce safety cost $\costsafe^1$. Furthermore, since the \gls{STP} policy has already attained the minimum of its individual cost $\costindi^1 =  \tcost^1$, the leader has no incentive to switch to another policy. Finally, due to the lexicographic ordering and symmetric safety costs, all followers arrive at an equilibrium
with their \gls{STP} policies that optimize $\tcostsafe^i$.

\p{Connection to generalized Stackelberg equilibrium}
Generalized Stackelberg equilibria~\cite{laine2023computation} encode pairwise safety with inequality constraint $h^{ij}_{\text{safe}} (\state^i_t,\state^j_t) \leq 0$. A generalized Stackelberg equilibrium can be interpreted as an \gls{LSE} of a game where~\eqref{eq:cost_assump} is such that
\begin{equation}
\label{eq:gen_Stbg}
    \stagecostsafe_t^{ij}(\state^i_t,\state^j_t) =\left\{\begin{array}{ll}
    0, & \text{if } h^{ij}_{\text{safe}} (\state^i_t,\state^j_t) \leq 0  \\
    \infty, & \text{otherwise}.
    \end{array}\right.
\end{equation}
In practice, one may use an optimization-based \gls{STP} solver (\eg, \gls{MPC}) to satisfy inequality constraints, and apply the same argument in \autoref{lem:courtesy} to show that \autoref{thm:STP} holds.

\subsection{\gls{STP} for Branch and Play}

We employ \gls{STP} as a subroutine of \gls{BNP} in Line~\ref{alg:cnp:bounding} of~\autoref{alg:bnp}.
We handle any incomplete permutation $\permincomplete = (\permincompleteassigned, \permincompleteunassigned)$ in two ways:
\begin{enumerate}[(a.)]
    \item Unassigned players play the last, avoiding all predecessors, \ie, $\preset^i = \permincompleteassigned,~\forall i \in \permincompleteunassigned$, 

    \item Unassigned players play unaware of others, \ie, $\preset^i = \emptyset,~\forall i \in \permincompleteunassigned$.
\end{enumerate}
We note that both options are useful in practice. The first option provides a tighter lower bound and is likelier to produce collision-free descendants that can be pruned. The second option generally requires less computational effort since the unassigned players only optimize their individual costs. While both are valid, the choice should be made depending on the specific application.

\p{Optimality}
\new{\gls{BNP} computes a globally-optimal solution (\autoref{thm:bnp}) whenever we can bound the permutations (incomplete and complete) with an exact (\ie, globally-optimal) subroutine that satisfies the admissibility property (\autoref{assump:admissibility}).
This implies that any globally-optimal trajectory optimizer (\eg, HJ Reachability~\cite{bansal2017hamilton}, graph search~\cite{hart1968formal,dijkstra2022note}, global optimization~\cite{houska2014branch,lee2011mixed}) satisfies the admissibility property, since more (non-negative) safety cost terms are added to the social cost as the order $\permorder$ increases. However, the currently available globally-optimal trajectory optimizers are computationally demanding. In contrast, local solvers like \gls{STP} are fast, but do not necessarily satisfy \autoref{assump:admissibility}.
Despite this limitation, we computationally assess that \gls{BNP} with \gls{STP} can recover solutions with low suboptimality in short computing times. In practice, we may manually enforce admissibility for local solvers through the following sufficient condition.
At an incomplete permutation $\permincomplete$, we set $\costsocial(\policy(\perm)) = \max\left(\costsocial(\policy(\perm)), \costsocial(\policy(\permincomplete))\right)$ for any $\perm \in \permmapping(\permincomplete)$. Although this procedure might lead to pruning some high-quality solutions, it also enables \gls{BNP} to support any local solver, and generally works well for the tested scenarios in \autoref{sec:result}.
}

\subsection{Practical Implementations}
\label{sec:practical}

We provide three examples of implementing \gls{STP} solvers (\textsc{SolveSTP} in Line~\ref{alg:stp:solve_stp} of~\autoref{alg:stp}).

\subsubsection{Expanded safety margin}
\label{sec:practical:margin}
Satisfaction of local optimality~\eqref{eq:assump:STP:best_f_1} is guaranteed with policies synthesized with differential dynamic programming (DDP)~\cite{mayne1973differential,murray1984differential}.
In practice, approximate DDP methods, such as \gls{ILQR}~\cite{li2004iterative}, are effective at finding a locally optimal solution.
A sufficient condition for enforcing~\eqref{eq:assump:STP:best_f_2} is the following:
\begin{enumerate}[(a.)]
\item The safety cost $\stagecostsafe^{ij}_t(\state^i_t,\state^j_t)$ is $0$ when there is no collision between player $i$ and $j$, and is positive otherwise, and 
\item Each follower ensures there is no collision with its predecessors.
\end{enumerate} 
For soft-constrained trajectory optimizers such as \gls{ILQR}, strict safety is generally not guaranteed even if policy $\policy^i$ locally optimizes $\tcost^i$~\cite[Theorem 17.3]{nocedal1999numerical}.
However, in practice, if the safety margin between each pair of agents is expanded, ~\cite{chen2021fastrack} reports an improved safe rate. This provides empirical planning robustness at the expense of increased conservativeness. To improve planning efficiency, we can also employ a principled algorithm~\cite{fridovich2018planning} that dynamically adjusts the safety margin while ensuring safety.

\subsubsection{Safety filters}
\label{sec:practical:sf}
To ensure strict safety, we can use a safety filter~\cite{hsu2023sf} to override, in a \emph{minimally invasive} fashion, a task policy $\tpolicy^{i}$ when necessary. This results in a \emph{filtered} policy $\tpolicy^{i,f}$ that yields a collision-free trajectory.

\begin{remark}
   Let the safety cost $\stagecostsafe_t^{ij}(\state^i_t, \state^j_t)$ attains its minimum (\eg, $\stagecostsafe_t^{ij}(\state^i_t, \state^j_t) = 0$) with a collision-free state pair $(\state^i_t, \state^j_t)$ and a greater value otherwise. A filtered policy $\tpolicy^{i,f}$ can be cast as a lexicographic Stackelberg equilibrium defined on the tuple $\left\langle \costsafe^i, \costindi^i\right\rangle$.
   It is not necessarily \gls{LSE} since the filtered policy may no longer satisfy~\eqref{eq:assump:STP:best_f_1}, \ie, $\tcost^i(\tpolicy^{i,f})$ can be suboptimal.
   We can use filter-aware planning approaches (\eg,~\cite{leung2020infusing,hu2022sharp,hu2023active,bejarano2023multi}) to optimize $\tcost^i$ while ensuring strict safety.
\end{remark}

\subsubsection{Constrained trajectory optimization}
\label{sec:practical:mpc}
We can employ constrained optimization methods (\eg, nonlinear \gls{MPC}~\cite{mayne2011tube, hu2018real, lopez2019dynamic, li2020robust}) as \gls{STP} solvers. Specifically, we can plan each agent's motion while avoiding the predecessors by encoding the safety specifications as \emph{hard constraints}.
Any locally optimal (and feasible) trajectory given by a constrained optimization-based \gls{STP} solver corresponds to a \emph{generalized} \gls{LSE} (c.f. \autoref{sec:guarantee_LSE}).

\begin{example}
    We use a combination of \gls{STP} designs in Sec.~\ref{sec:practical:margin} and~\ref{sec:practical:sf}.
    The \gls{STP} policy $\policy^i$ is given by an \gls{ILQR} planner using the $\ell_1$ penalty~\citep[Ch.17]{nocedal1999numerical} as the safety cost
    \begin{equation*}
        \stagecostsafe^{ij}_t(\state^i_t, \state^j_t; d) := \mu \max(d - \|p^i_t - p^j_t\|, 0),
    \end{equation*}
    where $\mu > 0$ is a large penalty coefficient, and the minimal separation $d$ is set to $0.4$, which is greater than the value (0.2) used for determining collision.
    With this expanded safety margin, we found, in most cases, that the computed trajectory is collision free.
    On top of this, we use a rollout-based safety filter~\cite[Sec.~3.3]{hsu2023sf} that overrides $\policy^i$ with a heuristic evasive strategy, which (empirically) gives zero collision rate, thereby yielding $\stagecostsafe^{ij}_t(\state^i_t, \state^j_t; 0.2) \equiv 0$.
\end{example}

%%%%%%%%%%%%%%%%%%%%%%%%%%%%%%%%%%%%%%%%%%%%%%%%%%%%%%%%%%%%%%%%v
\section{Experiments}
\label{sec:result}

We evaluate \gls{BNP} on three application domains: simulated \gls{ATC}, swarm formation in AirSim~\cite{shah2018airsim}, and hardware experiment of coordinating a delivery vehicle fleet.
The \gls{BNP}'s computations of the two simulated experiments are performed on a desktop with an AMD Ryzen $9$ $7950$X CPU.
The \gls{STP} solver uses \gls{ILQR}~\cite{li2004iterative} as the low-level single-agent trajectory optimizer; we implement \gls{ILQR} in JAX~\cite{jax2018github} for real-time performances. 
Since \gls{STP} computes players' policy sequentially, its computation time scales linearly with the number of agents, as we show in~\autoref{fig:time_stp}.
The code is available at \url{https://github.com/SafeRoboticsLab/Who_Plays_First}

\subsection{Air Traffic Control}
We consider a simulated \gls{ATC} setting where $\nagents \in \{4,5,6\}$ aircraft fly on collision trajectories.

\p{Baselines} We compare \gls{BNP} against the following baselines:
\begin{itemize}
    \item \emph{First-come-first-served (FCFS):} A higher order of play is assigned to an airplane that enters the \gls{ATC} zone earlier.
    \item \emph{Randomized:} We randomly assign the order of play.
    \item \emph{Nash ILQGame~\cite{fridovich2020efficient}:} Players use approximate Nash equilibrium strategies, which are agnostic of the orderings.
\end{itemize}
All methods use the same individual and safety costs for each player, and
\gls{BNP}, FCFS, and Randomized use the same \gls{STP} solver.
% timeout: 55 s
We consider three performance metrics. First, the \emph{closed-loop social cost}, \ie, the sum of all players' (normalized) running costs along the closed-loop trajectory. Second, the \emph{group flight time}, \ie, the time elapsed until the last aircraft reaches its destination. Third and last, the \emph{timeout rate}, the number of simulations where the group flight time exceeds $55$ seconds.
For each method, we performed $100$ randomized trials under the same random seed, each leading to a different initial state.
\autoref{tab:atc_results} summarizes the quantitative results.
Across different values of $\nagents$, \gls{BNP} outperformed the baselines in terms of all three metrics, and its performance improvement is magnified as the number of agents increases.
The Nash ILQ policy often resulted in poor coordination due to the lack of an external regulator that enforces an order of play to mitigate potential conflicts of interest among players.
In~\autoref{fig:atc_4_car} and~\autoref{fig:atc_8_car}, we show two representative trials with $N=4$ and $8$ airplanes to qualitatively demonstrate the trajectories computed by \gls{BNP}.

%%%%%%%%%%%%%%%%%%%%%%%%%%%%%%%%%%%%% ATC RESULTS %%%%%%%%%%%%%%%%%%%%%%%%%%%%%%%%%%%%%%%%%%%%%%%
\begin{table}[htbp!]
\centering
% \resizebox{14cm}{!}{
\resizebox{\iftoggle{preprint}{0.7\textwidth}{\columnwidth}}{!}{%
\begin{tabular}{lcrrrrrr}
\toprule
\textbf{\textcolor{Periwinkle}{Metric}} & 
\textcolor{Periwinkle}{$\bm{\nagents}$} &  
\textbf{\textcolor{Periwinkle}{FCFS}} &  
\textbf{\textcolor{Periwinkle}{Randomized}} &
\textbf{\textcolor{Periwinkle}{Nash ILQ}} &  
\textbf{\textcolor{Periwinkle}{\gls{BNP} (ours)}}   
\\ \midrule
% N = 4
Cost &    &1.24 $\pm$ 0.29    &1.19 $\pm$ 0.23   &3.61 $\pm$ 0.94    &\textbf{1.04 $\pm$ 0.24}         \\
Group (s) & 4
&25.76 $\pm$ 5.0    &24.98 $\pm$ 4.06    &44.50 $\pm$ 8.84    &\textbf{20.41 $\pm$ 2.72}    \\
T/O rate & 
&1\%    &\textbf{0\%}    &36\%    &\textbf{0\%}    \\ \midrule
% N = 5
Cost &            
&1.73 $\pm$ 0.49    &1.70 $\pm$ 0.43    &5.40 $\pm$ 1.36    &\textbf{1.41 $\pm$ 0.3}                \\
Group (s) & 5
&29.27 $\pm$ 8.29    &28.84 $\pm$ 6.74    &48.34 $\pm$ 11.4    &\textbf{20.5 $\pm$ 6.04}       \\
T/O rate &
&6\%    &2\%    &64\%    &\textbf{1\%}       \\ \midrule
% N = 6
Cost &            
&2.43 $\pm$ 0.82    & 2.55 $\pm$ 1.09    &  8.73 $\pm$ 2.13   & \textbf{2.0 $\pm$ 0.51}               \\
Group (s) & 6
&31.83 $\pm$ 9.34   & 33.62 $\pm$ 12.05   &  50.9 $\pm$ 11.92    & \textbf{22.18 $\pm$ 7.74}      \\
T/O rate & 
&9\%    &20\%    &96\%    &\textbf{2\%}       \\ \bottomrule
\end{tabular}
}
\vspace{0.1em}
\caption{
Case study of \gls{ATC}. Mean normalized closed-loop cost (Cost), group flight time (Group), and timeout rate (T/O rate) in $100$ randomized safe trials for different numbers of airplanes ($\nagents = 4,5,6$).
\gls{BNP} consistently outperforms the baselines for all metrics and number of agents.
}
\label{tab:atc_results}
\end{table}

%%%%%%%%%%%%%%%%%%%%%%%%%%%%%%%%%%%% ATC RESULTS %%%%%%%%%%%%%%%%%%%%%%%%%%%%%%%%%%%%%%%%%%%%%%%%
\begin{figure}[htbp!]
    \centering
    \includegraphics[width=\iftoggle{preprint}{0.5\textwidth}{0.9\columnwidth}]{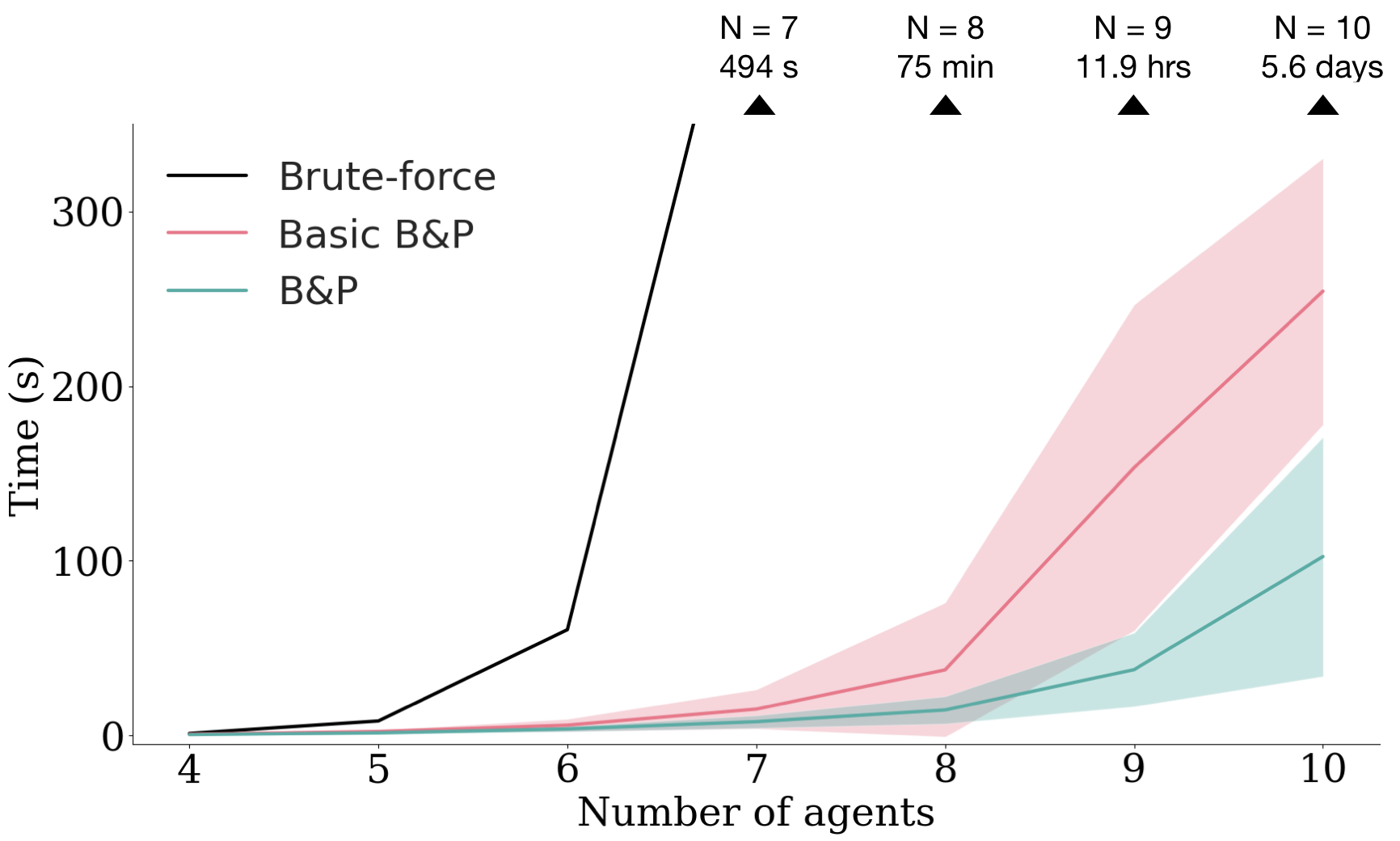}
    \caption{Computation time for obtaining the socially optimal permutation over $100$ randomized trials of the \gls{ATC} example. Solid lines and shaded areas represent the sample mean and standard deviation, respectively.
    Brute-force computation time for $\nagents \in [7,10]$ are indicated above the black triangles.}
    \label{fig:time_bnp}
\end{figure}

\begin{figure}[htbp!]
    \centering
    \includegraphics[width=\iftoggle{preprint}{0.5\textwidth}{0.9\columnwidth}]{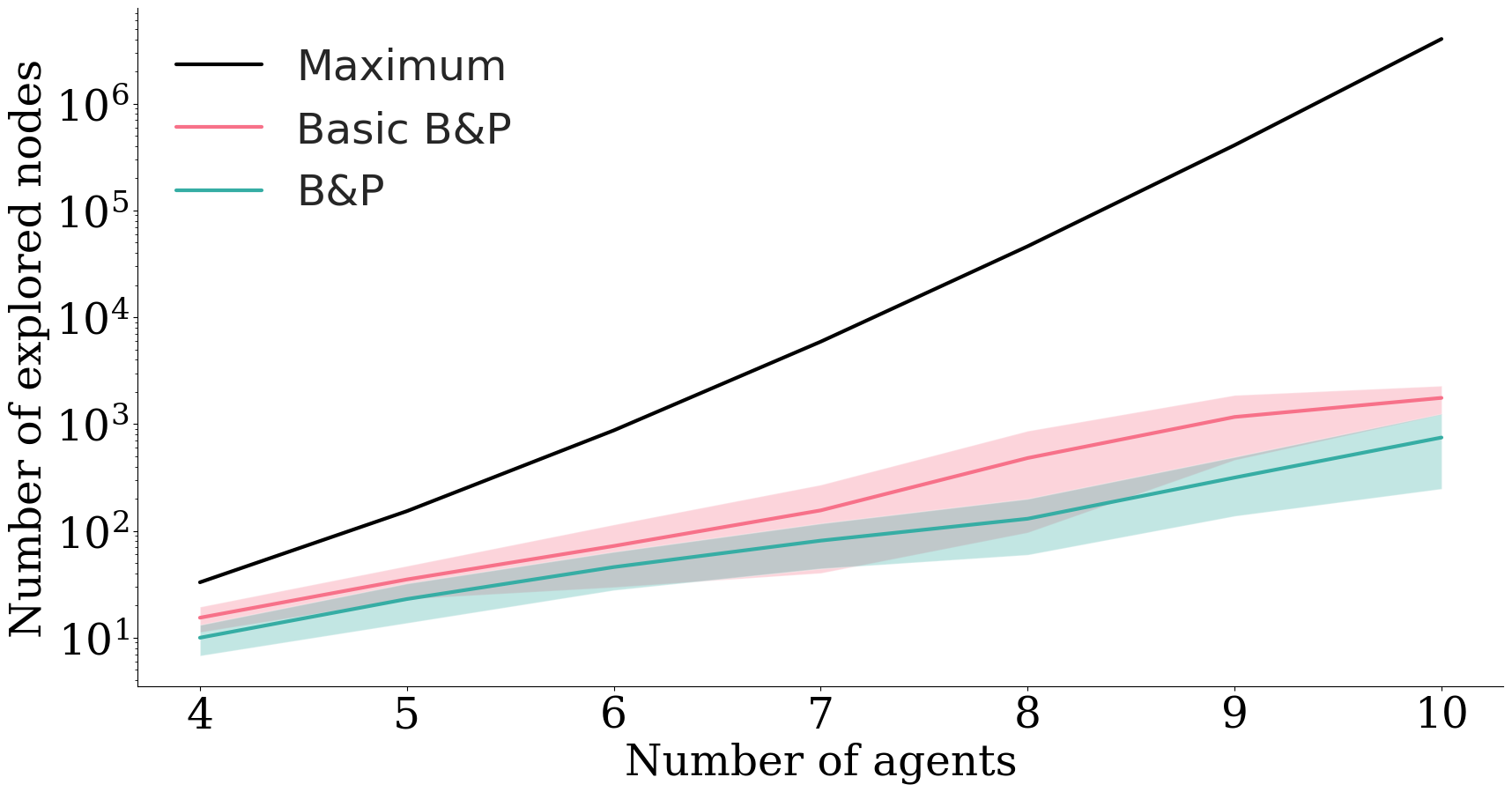}
    \caption{Number of explored nodes over $100$ randomized \gls{ATC} trials, plotted over the log scale. The solid black line represents the maximum number of explorable nodes.}
    \label{fig:expl_nodes}
\end{figure}

\begin{figure}[htbp!]
    \centering
    \includegraphics[width=\iftoggle{preprint}{0.5\textwidth}{0.9\columnwidth}]{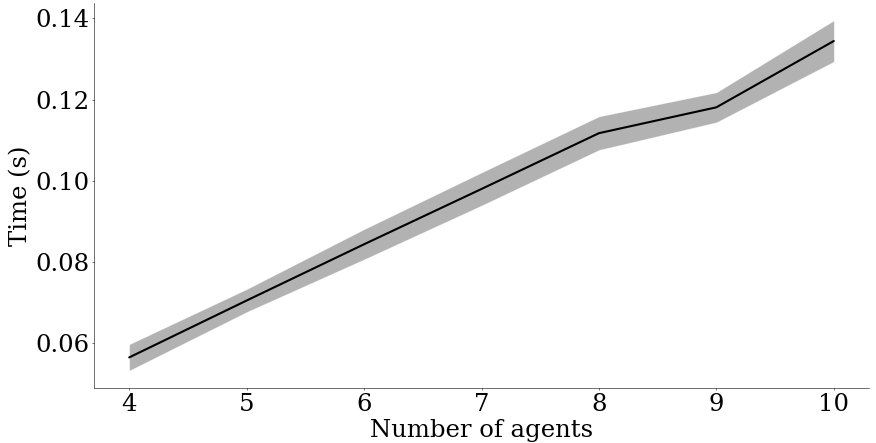}
    \caption{STP computation time over 100 randomized \gls{ATC} trials.}
    \label{fig:time_stp}
\end{figure}

\begin{figure*}[!hbtp]
    \centering
    \includegraphics[width=\iftoggle{preprint}{\textwidth}{1.85\columnwidth}]{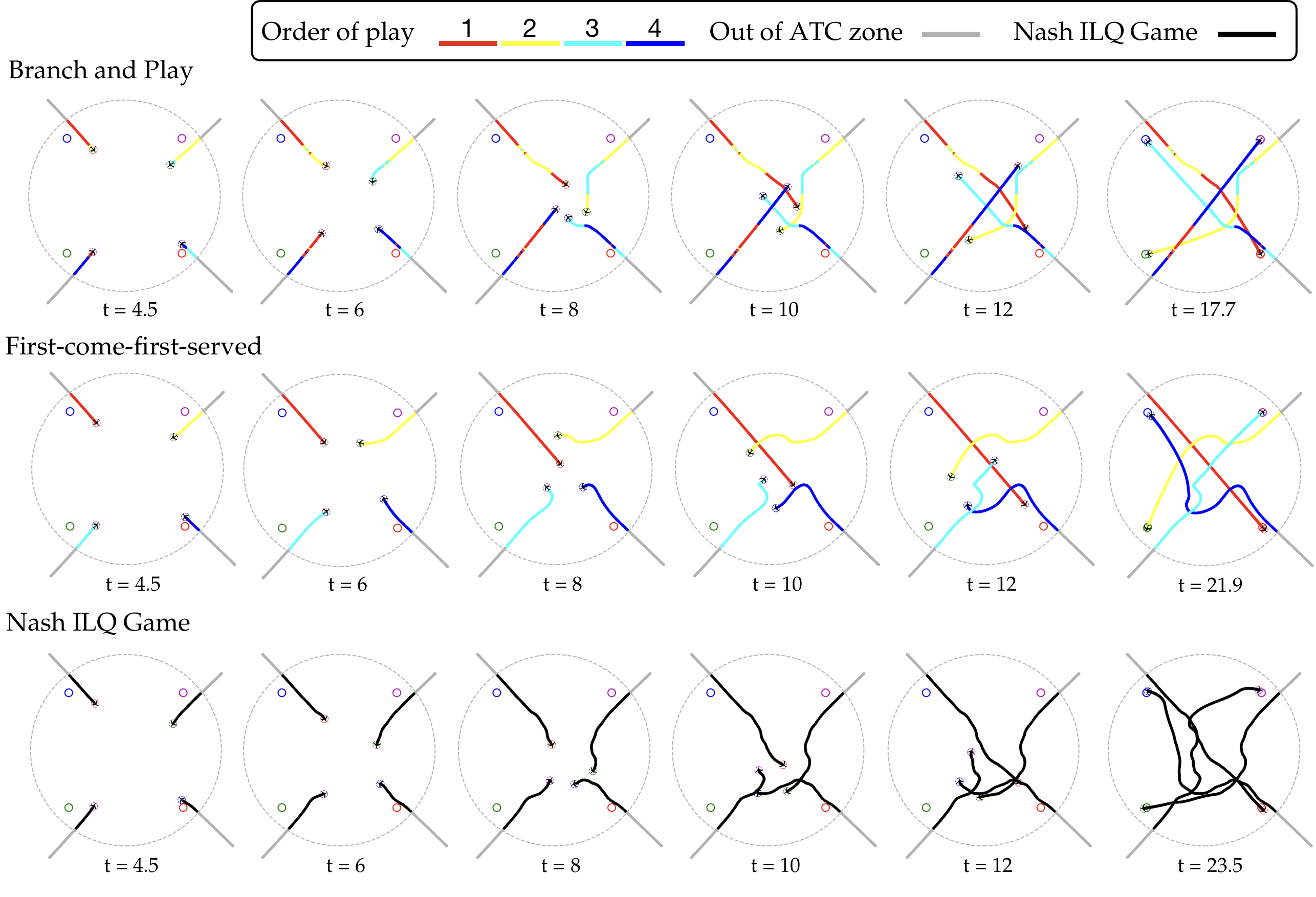}
    \caption{Case study: \gls{ATC} with $4$ airplanes initially flying on collision courses. Each agent's trajectory color represents its priority (warmer is higher). The \gls{ATC} zone is the dashed grey circle.
    \gls{BNP}~completed the task in $17.9$s with a normalized closed-loop social cost of $0.78$.
    \gls{FCFS} completed the task in $21.4$s ($+19.6 \%$ w.r.t. \gls{BNP}) with a normalized closed-loop social cost of $1.0$ ($+28.2\%$ w.r.t. \gls{BNP}).
    Nash ILQGame completed the task in $23.5$s ($+31.3 \%$ w.r.t. \gls{BNP}) with a normalized closed-loop social cost of $1.54$ ($+97.4\%$ w.r.t. \gls{BNP}). The trajectories are less smooth due to the lack of coordination.}
    \label{fig:atc_4_car}
\end{figure*}

\begin{figure*}[!hbtp]
    \centering
    \includegraphics[width=\iftoggle{preprint}{\textwidth}{1.95\columnwidth}]{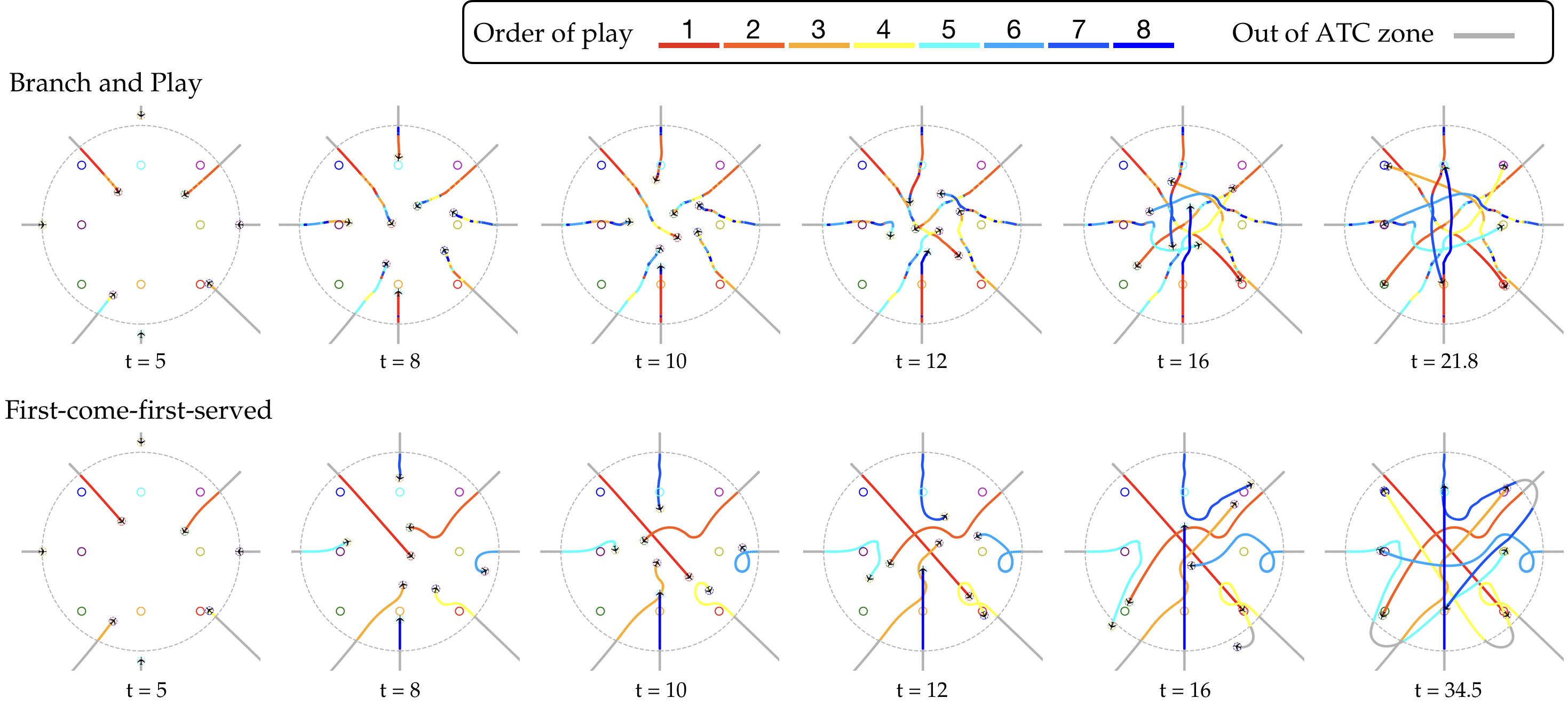}
    \caption{Case study: \gls{ATC} with $8$ airplanes initially flying on collision courses. Each agent's trajectory color represents its priority (warmer is higher). The \gls{ATC} zone is the dashed grey circle.
    All airplanes coordinated with \gls{BNP}~arrived at their destination in $21.8$s with a normalized closed-loop social cost of $2.73$.
    The \gls{FCFS} protocol resulted in severe congestion among the airplanes. Three lower-priority agents had to turn around and exit the \gls{ATC} zone in order to avoid colliding with higher-priority agents.
    As a result, it took $34.5$s ($+58.3 \%$ w.r.t. \gls{BNP}) for all airplanes to arrive at their destinations, incurring a normalized closed-loop social cost of $3.69$ ($35.2\%$ w.r.t. \gls{BNP}).
    Nash ILQ was not able to complete the planning task due to a timeout.}
    \label{fig:atc_8_car}
\end{figure*}

\begin{figure*}[!hbtp]
    \centering
    \includegraphics[width=\iftoggle{preprint}{\textwidth}{1.95\columnwidth}]{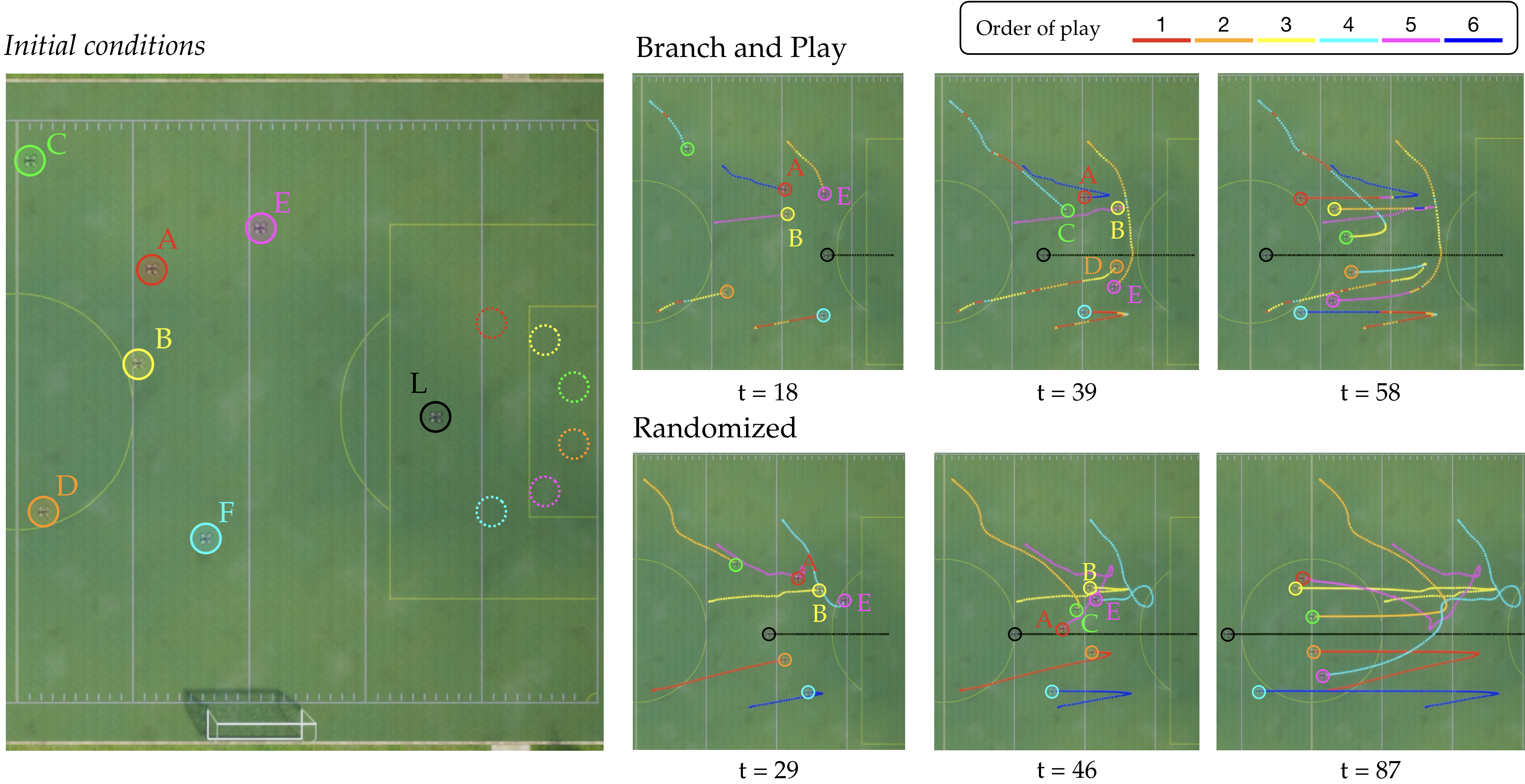}
    \caption{Case study: quadrotor swarm formation. Each agent's trajectory color represents its priority (warmer is higher). The hollow circles with dashed boundaries represent designated locations of task quadrotors in the half circle to be formed around the leading quadrotor. Compared to the \gls{STP} policy with a randomized permutation, \gls{BNP} provides a better trajectory in terms of formation time and reduced extra collision-avoidance maneuvers.}
    \label{fig:airsim}
\end{figure*}

Next, we investigate the practical scalability of \gls{BNP} increasing $\nagents$ beyond $6$.
To understand the benefit of pairwise-collision pruning introduced in~\autoref{sec:wpf:expl_pruning}, we consider a \textit{basic} version of \gls{BNP}, \new{which removes the collision-based pruning feature introduced in \cref{subsec:pairwise_prune} while keeping the native bound-based pruning.}
We also consider a brute force method that, in the worst case, has the same complexity of \gls{BNP}. \autoref{fig:time_bnp} shows that \gls{BNP} outperforms the two baselines and only uses a fraction of the time needed by the brute force method. 
In addition, with pairwise-collision pruning, we reduce, on average, the computation time by half.
In ~\autoref{fig:expl_nodes}, we show that the pruning strategy effectively reduces the search space, and the algorithm converges only having to explore half of the nodes on average compared to Basic \gls{BNP}.
Finally, we also report the computation time of the \gls{STP} subroutine for varying numbers of agents in~\autoref{fig:time_stp}.

\begin{figure}[!hbtp]
    \centering
    \includegraphics[width=\iftoggle{preprint}{0.5\textwidth}{0.8\columnwidth}]{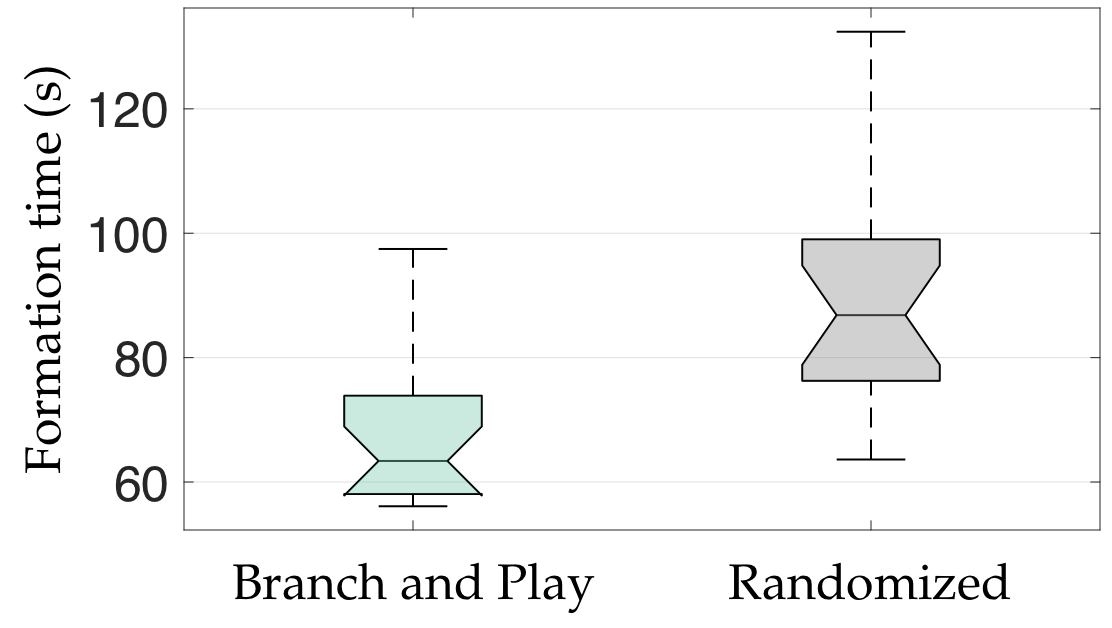}
    \caption{Time spent by \gls{BNP} and the randomized baseline to complete the swarm formation in $20$ randomized trials. Central marks, bottom, and top edges of the boxes indicate the median, $25$th, and $75$th percentiles, respectively. The maximum whisker length is set to $1.5$, which leads to $99.3\%$ coverage if the data is normally distributed.}
    \label{fig:box_plot}
\end{figure}

\begin{figure}[!hbtp]
    \centering
    \includegraphics[width=\iftoggle{preprint}{0.5\textwidth}{0.8\columnwidth}]{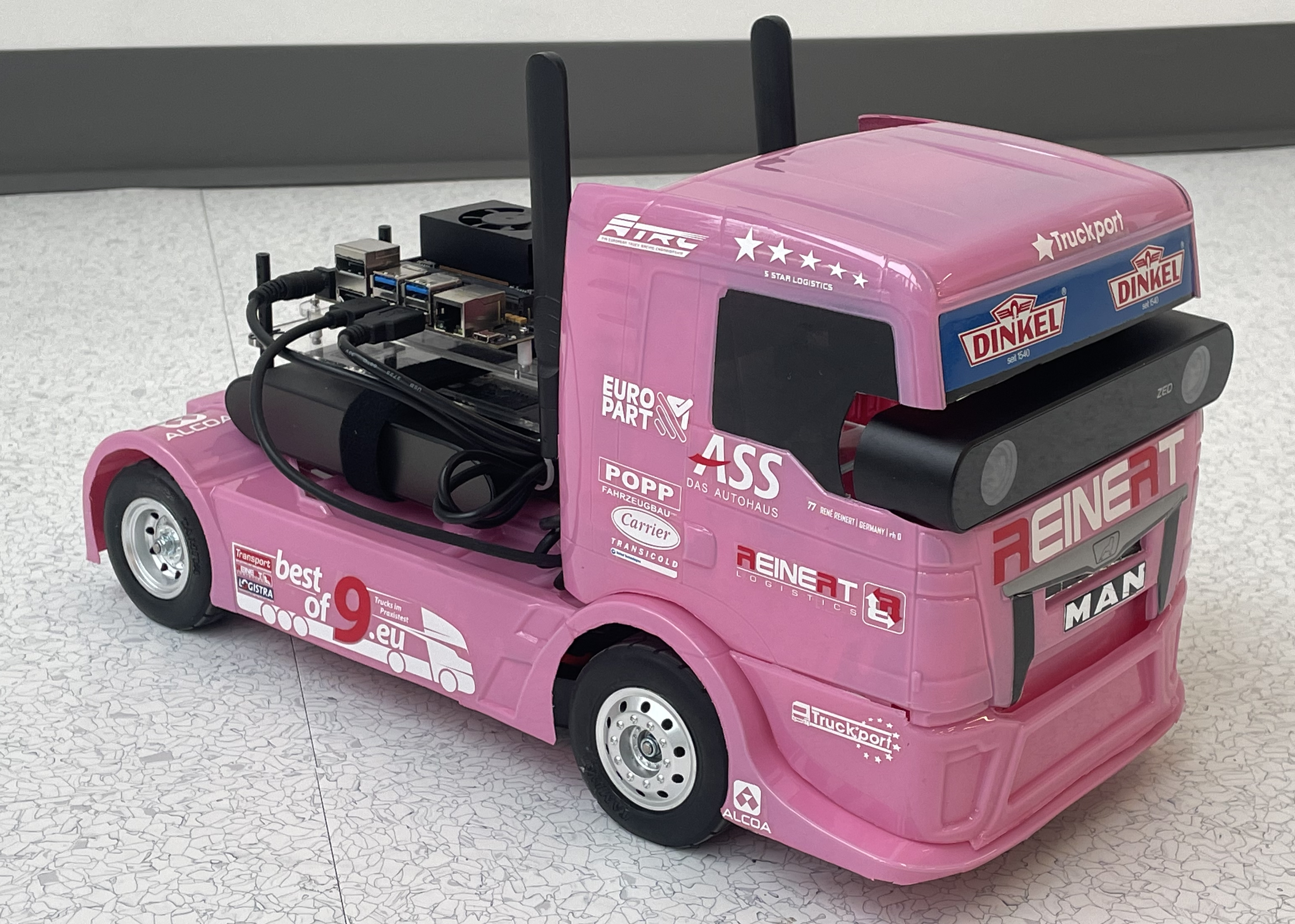}
    \caption{The autonomous miniature truck platform.}
    \label{fig:truck}
\end{figure}

\subsection{Quadrotor Swarm Formation}
In the second example, we apply \gls{BNP} to a quadrotor swarm formation task in AirSim~\cite{shah2018airsim}, a high-fidelity multi-agent simulator for drones. We respawn six \textit{task} quadrotors (A, B, C, D, E, and F) at random, non-collision locations within a soccer field (\autoref{fig:airsim}).
The goal of the quadrotors is to navigate towards a \textit{leading} agent (L) and form a half circle around it, where the locations of each task quadrotor in the half circle are predetermined at random.
Each task quadrotor is responsible for avoiding the leading quadrotor, which independently plans its trajectory towards a target location and communicates it with the task quadrotors.
Similar to the \gls{ATC} example, we restrict all quadrotors at the same altitude so that they plan on a horizontal space.
In \gls{STP}, each quadrotor plans with \gls{ILQR} using the double integrator model, producing a collision-free reference path and velocity profile, which are tracked by a carrot-following algorithm for low-level thrust control. \gls{BNP} and \gls{STP} run at $2$Hz and $10$Hz, respectively.

We show a representative trial in~\autoref{fig:airsim}.
At $t = 18$, \gls{BNP} assigned a high priority to E, who was distant from the formation. In contrast, A and B had a lower priority, as they were already close to their target positions.
Similarly, at a later time step $t=39$, A and B again yielded to C, who had not yet reached its destination. At the same time, E gave way to D to ensure the latter could enter its position uninterrupted. The formation was complete at $t=58$.

We also tested \gls{STP} with a random permutation, which resulted in less efficient coordination among the quadrotors.
In particular, at $t=29$, B intruded into the courses of A and E, forcing them to take evasive maneuvers and deviate from their nominal formation plans. The planning of A, which had a low priority, was further interrupted by agents C and E around $t=46$. The formation was not yet completed at $t=87$.

In~\autoref{fig:box_plot}, we represent the formation time from 20 trials with randomized initial states (both task and leading quadrotors).
The mean formation time with \gls{BNP} and the randomized permutations are $67.39$s and $89.11$s, respectively.

\begin{figure}[!hbtp]
    \centering
    \includegraphics[width=1.0\columnwidth]{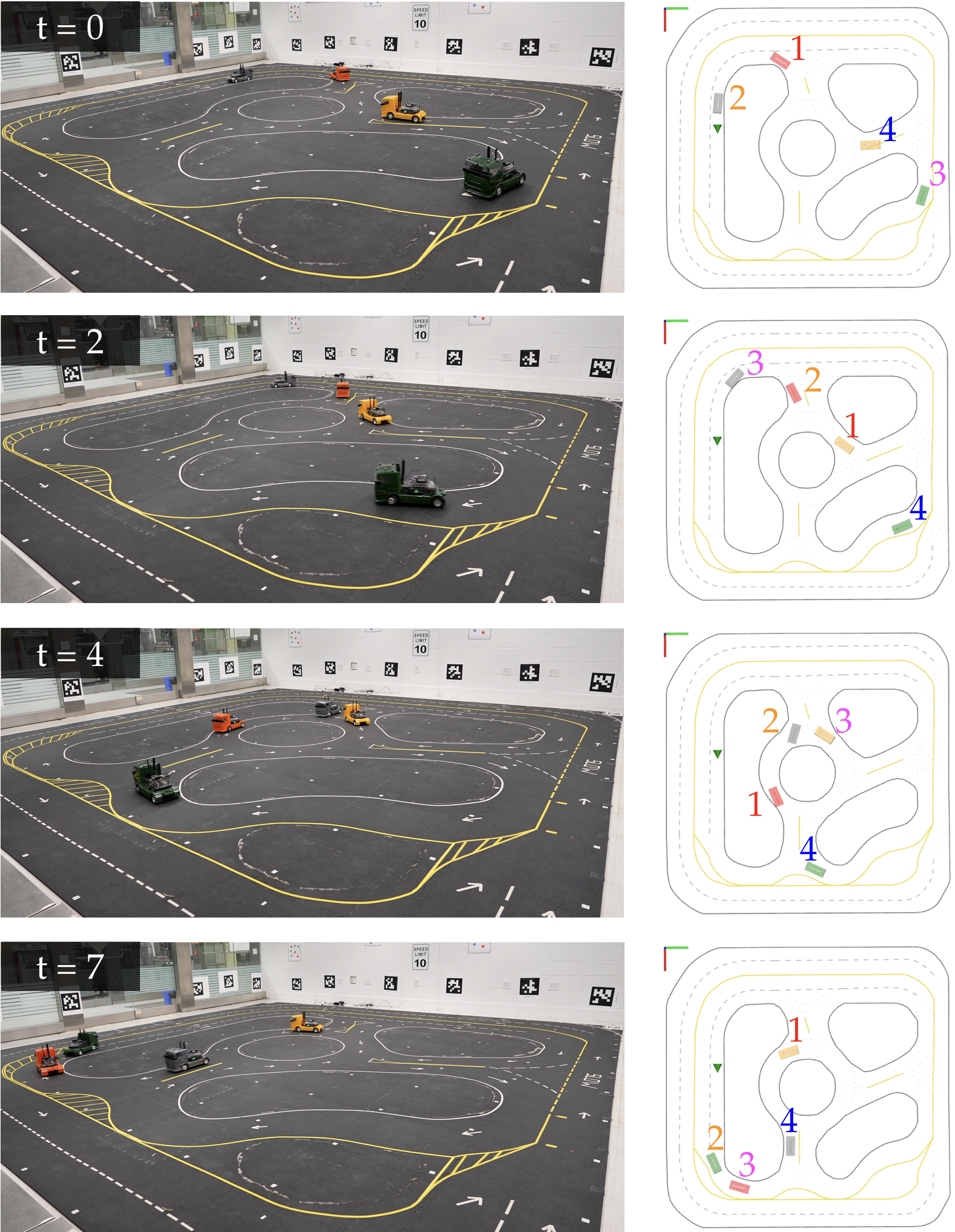}
    \caption{Case study: coordinating a delivery vehicle fleet at a roundabout. \emph{Left column:} Experiment snapshot at different time instances. \emph{Right column:} Corresponding birds-eye-view of the truck footprints on the map, and their orders of play assigned by \gls{BNP} indicated by the numbers near the footprints. All trucks share the same destination indicated by the green triangle.}
    \label{fig:truck_1}
\end{figure}

\begin{figure}[!hbtp]
    \centering
    \includegraphics[width=1.0\columnwidth]{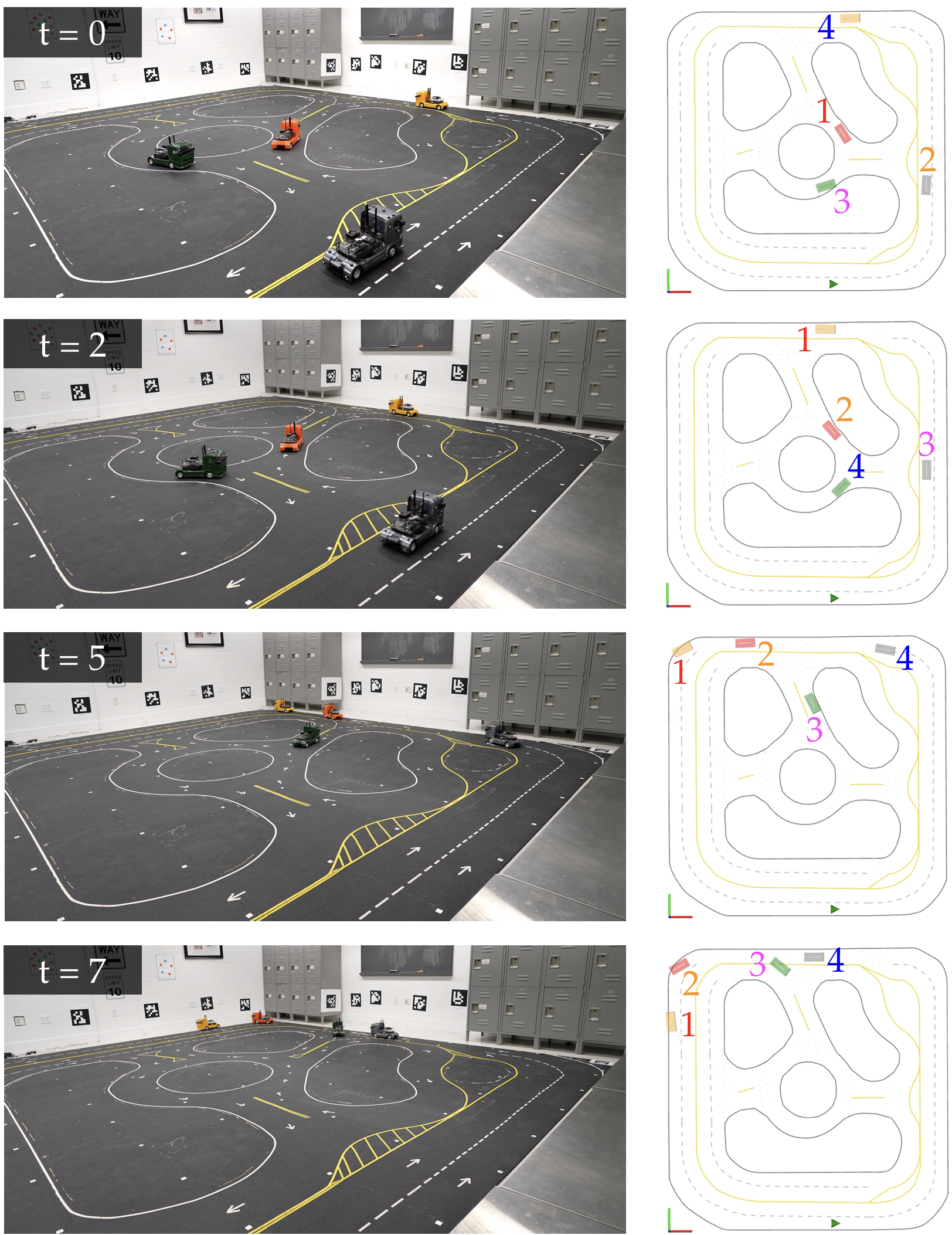}
    \caption{Case study: coordinating a delivery vehicle fleet to merge onto a highway (the outer loop). \emph{Left column:} Experiment snapshot at different time instances. \emph{Right column:} Corresponding birds-eye-view of the truck footprints on the map, and their orders of play assigned by \gls{BNP} indicated by the numbers near the footprints. All trucks share the same destination indicated by the green triangle.}
    \label{fig:truck_2}
\end{figure}

\subsection{Hardware Demonstration on a Delivery Vehicle Fleets}
In the last example, we employ \gls{BNP} for a hardware experiment where we attempt to coordinate an autonomous delivery vehicle fleet (\autoref{fig:truck}) in a scaled metropolis.
Our experimental platform consists of multiple homemade mini-trucks built on a $1/10$ RC chassis and equipped with the Nvidia Jetson Xavier NX for onboard computation of a localization algorithm and \gls{STP}.
We use a custom visual-inertial SLAM algorithm along with the AprilTag \cite{wang2016apriltag} detection to localize each vehicle in a pre-defined road map (right columns of \autoref{fig:truck_1} and~\autoref{fig:truck_2}).
Each truck runs \gls{STP} with an ILQR trajectory planner to track a reference path. It also respects the road boundaries, and avoids collisions with the leading vehicles.
The reference path is found as the shortest path (subject to traffic rules, \eg, the direction of a road) between the truck's initial and target positions using a graph search algorithm.
\gls{BNP} computations are carried out on a desktop with an Intel i$7$-$7700$K CPU.
All trucks receive the order of play and broadcast their positions and planned trajectory through the robot operating system (ROS).
\autoref{fig:system} provides a picture of the setup.

\p{Roundabout}
In the first experiment (\autoref{fig:truck_1}), 
$\nagents=4$ trucks interacted near a roundabout: $3$ of them entered the roundabout (the yellow, red, and grey trucks) while the green one drove on a side road and interacted with the red truck near the exit of the roundabout.
We initialized \gls{BNP} with a random permutation, which was deemed optimal since there was no collision detected among all agents' plans ($2$s lookahead) at $t=0$.
During $t \in [2, 4]$, \gls{BNP} transitioned from prioritizing the yellow vehicle to letting it yield to the red and grey trucks, which were about to enter the roundabout. 
Although contrary to common traffic rules, this unusual yielding improved the efficiency of the fleet as a whole, since, otherwise, both the red and grey trucks would have waited until the yellow truck passed.
Between $t \in [4,7]$, the priority between the red and green trucks swapped--a natural decision as the green truck had already led the red one towards the target (indicated by the green triangle). 
In the meantime, the yellow truck, now the last in the platoon, was assigned the highest priority to speed up and avoid further delays.

\p{Highway merging}
In the second experiment (\autoref{fig:truck_2}), we consider a traffic-merging scenario when all trucks were routed to drive along a highway (the outer loop) toward the destination.
At $t=0$, two trucks (yellow and grey) were already on the highway while the other two (red and green) were about to merge onto the highway from the roundabout. Similarly to the roundabout experiment, we initialized \gls{BNP} with a random permutation.
At $t=2$, the priority between the yellow and red trucks switched, which ensured the former vehicle cruising along the highway uninterrupted.
Three seconds later, \gls{BNP} realized that the green vehicle could exit the roundabout and merge onto the highway before the grey vehicle reached the merging point.
Therefore, the priority between the two trucks was swapped.
At $t=7$, the grey vehicle yielded to the green one, which then safely merged onto the highway, a platoon was formed, and the coordination was complete.

%%%%%%%%%%%%%%%%%%%%%%%%%%%%%%%%%%%%%%%%%%%%%%%%%%%%%%%%%%%%%%%%
\section{Limitations and Future Work}
\label{sec:limitation}

In this paper, we consider a restricted class of multi-robot planning problems in which the interactions among agents only concern collision-related safety---an assumption required by \gls{STP}.
We expect \gls{BNP} to unleash its full potential with more versatile Stackelberg subroutines that support more generic cost structures.
Similarly, the integration of a data-driven or learning-based Stackelberg subroutine can potentially enable \gls{BNP} to scale even to larger problems. As the number of agents increases, optimization-based Stackelberg solvers (\eg, the \gls{ILQR}-based \gls{STP} implementation) slow down inevitably, while a learning-based one might not.
Finally, we see an open opportunity to investigate more sophisticated, and even application-specific, branching, exploration, and pruning strategies.
In particular, our current implementation of \gls{BNP} tree search is based on sequentially solving \gls{STP} subproblems.
We may employ parallelism for solving \gls{STP}---a common practice in tree-based optimization---to further reduce the computation time.
With the above improvements, we expect to significantly improve the computation performance of \gls{BNP}.

%%%%%%%%%%%%%%%%%%%%%%%%%%%%%%%%%%%%%%%%%%%%%%%%%%%%%%%%%%%%%%%%
\section{Conclusion} 
\label{sec:conclusion}
In this paper, we introduced \gls{BNP}, a novel algorithm that efficiently computes the socially optimal order of play for $N$-player Stackelberg trajectory games. As a subroutine to \gls{BNP}, we extended \gls{STP} to compute a valid local Stackelberg equilibrium.
We demonstrated the versatility and efficiency of \gls{BNP} by an extensive experiment campaign on air traffic control, quadrotor swarm formation, and coordination of an autonomous delivery truck fleet.
\gls{BNP} consistently outperformed the baselines, and provided solutions with better robustness and, importantly, optimal social costs.

\balance
\bibliographystyle{IEEEtran}
\bibliography{IEEEabrv,references.bib}

\newpage
\appendix

%%%%%%%%%%%%%%%%%%%%%%%%%%%%%%%%%%%%%%%%%%%%%%%%%%%%%%
\iftoggle{preprint}{
\section{List of Acronyms}
}
{
\subsection{List of Acronyms}
}
\renewcommand{\glossarysection}[2][]{}
\setglossarystyle{mylist}
\renewcommand*{\glsgroupskip}{}
\printglossary[]

\begin{figure*}[tp]
    \centering
    \iftoggle{preprint}{
    \includegraphics[width=0.4\columnwidth]{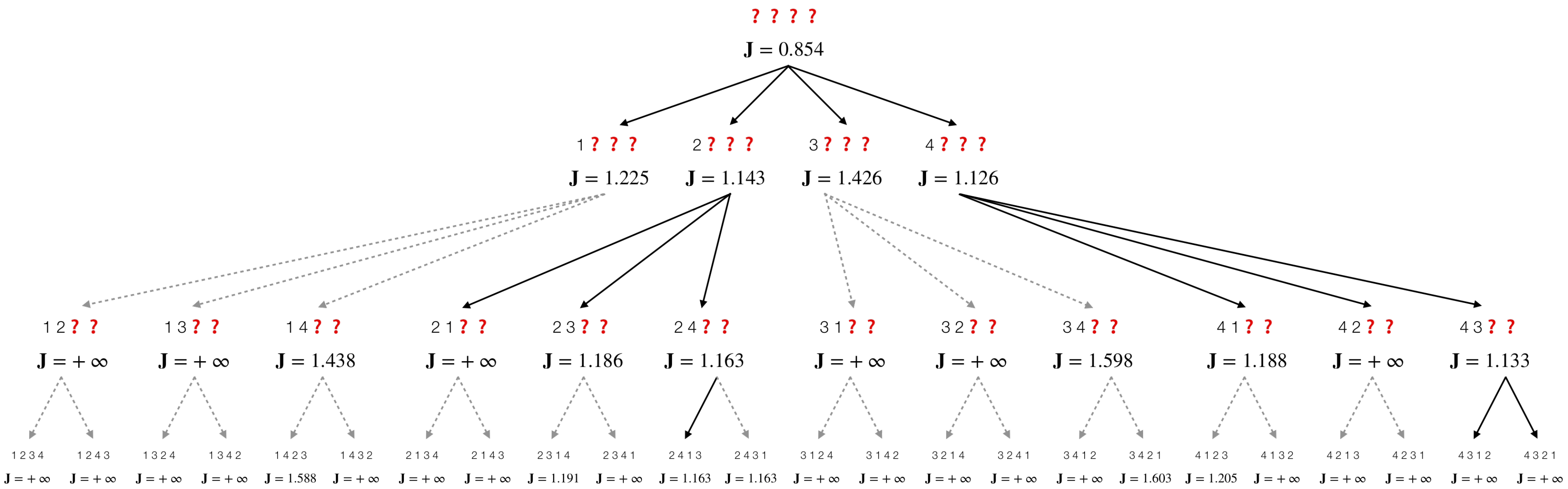}
    }
    {
    \includegraphics[width=2.0\columnwidth]{figs/bnp_tree_full.pdf}
    }
    \caption{The full \gls{BNP} search tree in \autoref{fig:bnp_tree}. Dashed grey lines indicate pruned branches. $\costsocial = +\infty$ corresponds to cases when \gls{STP} fails to compute a collision-free \gls{LSE}.}
    \label{fig:bnp_tree_full}
\end{figure*}

%%%%%%%%%%%%%%%%%%%%%%%%%%%%%%%%%%%%%%%%%%%%%%%%%%%%%%
\iftoggle{preprint}{
\section{Aircraft Model for Planning}
}{
\subsection{Aircraft Model for Planning}
}
\label{apdx:model}
We assume the airplanes cannot change altitude.
To this end, we model each airplane with the unicycle model:
\begin{equation*}
    \begin{aligned}
    \dot{p}^i_x & =v^i \cos \left(\theta^i\right) \\
    \dot{p}^i_y & =v^i \sin \left(\theta^i\right) \\
    \dot{v}^i &= a^i \\
    \dot{\theta}^i & =\omega^i
    \end{aligned}
\end{equation*}
where $p^i := (p_x^i, p_y^i)$ is the position, $v^i$ is the speed in the body frame, and $\theta^i$ is the heading in the global coordinate. 
The control inputs include the acceleration $a^i$ and the turn rate $\omega^i$.
We time-discretize the model with the forward Euler method and use it subsequently in planning.

%%%%%%%%%%%%%%%%%%%%%%%%%%%%%%%%%%%%%%%%%%%%%%%%%%%%%%
\iftoggle{preprint}{
\section{Proof of~\autoref{lem:courtesy}}
}{
\subsection{Proof of~\autoref{lem:courtesy}}
}
\label{apdx:lem:courtesy}
\begin{proof}
    We start with $N=2$. Suppose there exists a $\tpolicy^1 \in \tpolicyset^1(\policy^1)$ such that $\cost^1(\tpolicy^{1}(\state^1_0)) < \cost^1(\policy^{1}(\state^1_0))$.
    Since the strategy of the leader $\policy^1$ optimizes $\tcost^1 = \costindi^1$, then we have that $\costsafe^1(\ttraj^1(\tpolicy^1), \ttraj^2(\tpolicy^2(\tpolicy^1))) < \costsafe^1(\traj^1(\policy^1), \traj^2(\policy^2(\policy^1)))$, where $\tpolicy^2(\tpolicy^1)$ denotes the follower's new policy in response to the leader's policy switch.
    This, however, implies that $\costsafe^2(\ttraj^1(\tpolicy^1), \ttraj^2(\tpolicy^2(\tpolicy^1))) < \costsafe^2(\traj^1(\policy^1), \traj^2(\policy^2(\policy^1)))$ since the safety costs are symmetric, contradicting~\eqref{eq:assump:STP:best_f_2}.
    For $N=3$, we know from the above reasoning that player $2$ will not switch policy when the leader deviates, thus yielding the same state trajectory.
    Therefore, the problem reduces to the $N=2$ case by viewing the leader and player $2$ jointly as a single player.
    The same result holds for $N>3$ by induction.
\end{proof}

%%%%%%%%%%%%%%%%%%%%%%%%%%%%%%%%%%%%%%%%%%%%%%%%%%%%%%
\iftoggle{preprint}{
\section{Empirical Validation of Assumptions and \autoref{thm:bnp}}}
{
\subsection{Empirical Validation of Assumptions and \autoref{thm:bnp}}
}
\new{We provide in \autoref{fig:bnp_tree_full} the full \gls{BNP} search tree for the $4$-player \gls{ATC} example in \autoref{fig:bnp_tree}.
In this example, we numerically verified that admissibility (\autoref{assump:admissibility}) and best-effort cautious follower (\autoref{assump:STP:best_f}) are both satisfied for all subgames.
As a result, \gls{BNP} was able to find the socially-optimal Stackelberg equilibrium while pruning half of the branches in the tree, thus verifying \autoref{thm:bnp}.
}

\end{document}